\documentclass{article}

\newif\ifarxiv 
\arxivtrue
\newif\ifneurips

\ifneurips
\usepackage{neurips_2026}
\fi
\usepackage{enumitem}

\usepackage[utf8]{inputenc} 
\usepackage[T1]{fontenc}    
\usepackage{hyperref}       
\usepackage{url}            
\usepackage{booktabs}       
\usepackage{amsfonts}       
\usepackage{nicefrac}       
\usepackage{microtype}      
\usepackage{xcolor}         

\usepackage{natbib}
\PassOptionsToPackage{dvipsnames}{xcolor}
\PassOptionsToPackage{hypertexnames=false}{hyperref}

\usepackage{amsmath, amssymb, amsthm, mathtools}

\newtheorem{theorem}{Theorem}[section]
\newtheorem{claim}{Claim}[section]
\newtheorem{lemma}{Lemma}[section]
\newtheorem{proposition}{Proposition}[section]

\newtheorem{definition}{Definition}[section]
\newtheorem{assumption}{Assumption}[section]

\theoremstyle{remark}
\newtheorem{remark}[theorem]{Remark}

\newtheorem*{theorem*}{Theorem}
\newtheorem*{lemma*}{Lemma}
\newtheorem*{proposition*}{Proposition}
\newtheorem*{claim*}{Claim}
\newtheorem*{corollary*}{Corollary}
\newtheorem*{definition*}{Definition}
\newtheorem*{remark*}{Remark}

\newcommand{\ar}[1]{\textcolor{brown}{[AR: #1]}}

\usepackage{appendix}

\usepackage[utf8]{inputenc} 

\usepackage[T1]{fontenc}    
\usepackage{hyperref}       

\usepackage{url}            
\usepackage{booktabs}       
\usepackage{amsfonts}       
\usepackage{latexsym,amssymb,amscd,amsopn,amsmath}
\usepackage{bbm}
\usepackage{mathrsfs}
\usepackage{nicefrac}       
\usepackage{microtype}      
\usepackage{xcolor}         
\usepackage{algorithm}
\usepackage{verbatim}
\usepackage[noend]{algpseudocode}

\usepackage{xspace}
\usepackage{mathtools}
\usepackage{svg}

\usepackage{makecell}
\usepackage{aliascnt}

\usepackage[nameinlink,capitalize]{cleveref}
\usepackage{amsthm}
\usepackage{geometry}[1in]

\usepackage{thmtools}

\definecolor{dgreen}{rgb}{0,0.5,0}
\hypersetup{
  colorlinks=true,
  linkcolor=blue,
  filecolor=blue,
  citecolor = dgreen,      
  urlcolor=cyan,
}

\numberwithin{theorem}{section}
\numberwithin{lemma}{section}
\numberwithin{definition}{section}

\newcommand{\nc}{\newcommand}
\nc{\DMO}{\DeclareMathOperator}
\newcount\Comments  
\DeclareMathOperator*{\argmin}{arg\,min} 
\DeclareMathOperator*{\argmax}{arg\,max}

\nc{\Moracle}{\MM^{\mathsf{oracle}}}
\nc{\tilMoracle}{\til{\MM}^{\mathsf{oracle}}}
\nc{\SimulateReduction}{\texttt{SimulateReduction}\xspace}
\nc{\TestReduction}{\texttt{DistinguishReduction}\xspace}
\nc{\SimulateSampling}{\texttt{SimulateSampling}\xspace}
\nc{\SimulateRegression}{\texttt{SimulateRegression}\xspace}
\nc{\Osample}{{\MO_{\mathsf{samp}}}}
\nc{\Oregress}{{\MO_{\mathsf{regress}}}}
\nc{\Oregressp}{{\MO'_{\mathsf{regress}}}}
\nc{\Obandits}{{\MO_{\mathsf{bandits}}}}
\nc{\dom}{\mathsf{dom}}
\nc{\SF}{\mathscr{F}}
\nc{\Fchernoff}{\MF^{\mathsf{chernoff}}}

\DMO{\prox}{prox}
\DMO{\Span}{span}
\DMO{\UCB}{UCB}
\DMO{\LCB}{LCB}
\nc{\expl}[2]{\ME^{#1}_{#2}}
\nc{\tilmdp}[1]{\til \MM({#1})}
\nc{\barpdp}[2]{\ol \MP_{#1}({#2})}
\nc{\barmdp}[1]{\ol \MM({#1})}
\nc{\hatmdp}[1]{\wh \MM({#1})}
\nc{\rem}[2]{\MR_{#1}({#2})}
\nc{\Pigen}{\Pi^{\rm gen}}
\nc{\Pidet}{\Pi^{\rm det}}
\nc{\PiZ}{\Pi_{\SZ}^{\rm markov}}
\nc{\und}[3]{\MU_{{#1}}^{{#2}}({#3})}
\nc{\zlow}[2]{\MZ_{{#1}}^\lowv({#2})}
\nc{\dg}{\dagger}
\nc{\bB}{\mathbf{B}}
\nc{\unif}{\mu_{\rm unif}}
\nc{\indsig}[2]{\mathcal{I}_{#1}({#2})}
\nc{\total}{{\rm fin}}
\nc{\early}{{\rm pre}}
\nc{\zsink}{z_{\rm sink}}
\nc{\lowv}{{\rm low}}
\nc{\oo}[1]{\texttt{o}({#1})}
\nc{\posnrm}[1]{\left[ {#1} \right]_+}
\nc{\negnrm}[1]{\left[ {#1} \right]_-}
\nc{\tvnrm}[1]{\left\| {#1} \right\|_1}
\nc{\absval}[1]{\left| {#1} \right|}
\nc{\normalize}[1]{\mathfrak{n}\left({#1}\right)}

\nc{\SZ}{\textsf{Z}}
\nc{\SO}{\textsf{O}}
\nc{\suff}[2]{{\rm suff}_{#1}({#2})}
\nc{\UPhi}{\mathscr{U}_{X,H}}
\nc{\UPhis}{\til{\mathscr{U}}_{X,H,\MF}}
\nc{\SV}{\mathscr{V}}
\nc{\Phiset}{\Phi_{X,H}}
\nc{\Phisets}{\til{\Phi}_{X,H,\MF}}
\nc{\Lyu}{{\mathtt{Lyu}}}
\nc{\wAlg}{{\widetilde \Alg}}

\nc{\ApproxMDP}{\texttt{ConstructMDP}\xspace}
\nc{\mainalg}{\texttt{BaSeCAMP}\xspace} 
\nc{\bspanner}{\texttt{BarySpannerPolicy}\xspace}

\nc{\gamvec}{\gamma}
\nc{\til}{\widetilde}
\nc{\td}{\tilde}
\nc{\wh}{\widehat}
\nc{\old}[1]{\ifnum\Comments=1 {\color{brown}  [COPIED: #1]}\fi}
\definecolor{darkgreen}{rgb}{0.0, 0.5, 0.0}
\nc{\noah}[1]{\ifnum\Comments=1 {\color{darkgreen} [ng: #1]}\fi}
\nc{\dhruv}[1]{\ifnum\Comments=1 {\color{purple} [dr: #1]}\fi}
\nc{\BP}{\mathbb{P}}
\nc{\BM}{\mathbb{M}}
\nc{\bbapx}{\bb^{\rm apx}}
\nc{\bbapxs}[1]{\bb^{\rm apx, {#1}}}

\nc{\fools}[3]{\MF_{#3}({#1}, {#2})}
\nc{\fool}[2]{\MF({#1},{#2})}
\nc{\clip}[2]{{\rm clip}\left[ \left. {#1} \right| {#2} \right]}
\nc{\imax}{\omega}
\DMO{\conv}{conv}
\nc{\MH}{\mathcal{H}}
\nc{\CH}{\mathscr{H}}
\nc{\CB}{\mathscr{B}}
\nc{\cD}{\mathscr{D}}
\nc{\MC}{\mathcal{C}}
\nc{\st}{\star}

\nc{\lng}{\langle}
\nc{\rng}{\rangle}
\DMO{\OOPT}{opt}
\nc{\dopt}[2]{\ell_{\OOPT}({#1},{#2})}
\nc{\grad}{\nabla}
\nc{\MG}{\mathcal{G}}
\nc{\MP}{\mathcal{P}}
\nc{\PP}{\mathbb{P}}
\nc{\TT}{\mathbb{T}}
\nc{\TTmax}{\TT_{\max}}
\DMO{\Ham}{Ham}
\DMO{\Gap}{Gap}
\DMO{\GD}{GD}
\DMO{\GDA}{GDA}
\DMO{\EG}{EG}
\DMO{\OGDA}{OGDA}
\DMO{\Unif}{Unif}
\DMO{\Tr}{Tr}
\nc{\ul}{\underline}
\nc{\ol}{\overline}
\nc{\Qu}{\ul{Q}}
\nc{\Qo}{\ol{Q}}
\nc{\Ro}{\ol{R}}
\nc{\Vu}{\ul{V}}
\nc{\Vo}{\ol{V}}
\nc{\RanQ}{\Delta Q}
\nc{\RanV}{\Delta V}
\nc{\clipQ}{\Delta \breve{Q}}
\nc{\frzQ}{\Delta \mathring{Q}}
\nc{\clipV}{\Delta \breve{V}}
\nc{\clipdelta}{\breve{\delta}}
\nc{\cliptheta}{\breve{\theta}}
\nc{\delmin}{\Delta_{{\rm min}}}
\nc{\delmins}[1]{\Delta_{{\rm min},{#1}}}
\nc{\gapfinal}[1]{\max \left\{ \frac{\frzQ_{{#1}}^{k^\st}(x,a)}{2H}, \frac{\delmin}{4H} \right\}}
\nc{\post}[2]{R({#1}; {#2})}
\nc{\posts}[3]{R_{#3}({#1}; {#2})}
\nc{\MAJ}{\mathsf{MAJ}}
\nc{\Dnull}{D^{\circ}}

\nc{\BKW}{\mathtt{BKW}}
\nc{\Dec}{\mathtt{Dec}}
\nc{\delreg}{\delta_{\mathsf{reg}}}
\nc{\delreal}{\delta_{\mathsf{real}}}
\nc{\Sreg}{S_{\mathsf{Reg}}}
\nc{\Treg}{T_{\mathsf{Reg}}}
\nc{\PC}{\mathtt{PC}}
\nc{\alphaPC}{\alpha_{\mathsf{PC}}}
\nc{\TPC}{T_{\mathsf{PC}}}
\nc{\SPC}{S_{\mathsf{PC}}}
\nc{\delsmall}{{\delta_{\mathsf{small}}}}
\nc{\CL}{\mathtt{ContrastLearn}}
\nc{\Select}{\mathtt{Select}}
\nc{\piunif}{{\pi_{\mathsf{unif}}}}
\nc{\picov}{{\pi_{\mathsf{cov}}}}
\nc{\ZZ}{\mathbb{Z}}
\nc{\sk}{{\mathsf{sk}}}
\nc{\Enc}{\mathtt{Enc}}
\nc{\EntLPN}{\mathtt{EntangleLPN}}
\nc{\mureg}{{\mu_{\mathsf{reg}}}}
\nc{\PPE}{\mathtt{PPE}}
\nc{\FQI}{\mathtt{FQI}}
\nc{\False}{\mathtt{False}}
\nc{\True}{\mathtt{True}}
\nc{\epreg}{\epsilon_{\mathsf{reg}}}
\nc{\algnst}[1]{\begin{align*}#1\end{align*}}
\nc{\algn}[1]{\begin{align}#1\end{align}}
\nc{\matx}[1]{\left(\begin{matrix}#1\end{matrix}\right)}

\nc{\pimix}{{\pi_{\mathsf{mix}}}}
\nc{\BPC}{{B_{\mathsf{PC}}}}
\nc{\size}{\mathrm{size}}
\nc{\OLIVE}{\texttt{OLIVE}}
\nc{\RP}{\textsf{RP}}
\nc{\OPT}{\mathrm{OPT}}
\nc{\cprp}{c_{\mathsf{PRP}}}
\nc{\Mtoy}{\MM_{\mathsf{toy}}}
\nc{\Brute}{\mathtt{Brute}}

\nc{\nuu}{\nu}

\nc{\bel}[1]{\mathbf{b}({#1})}
\nc{\nbel}[1]{\bar{\mathbf{b}}({#1})}
\nc{\sbel}[2]{\mathbf{b}'_{#1}({#2})}
\nc{\nsbel}[2]{\bar{\mathbf{b}}'_{#1}({#2})}

\nc{\bv}{\mathbf{v}}
\nc{\bone}{\mathbf{1}}
\nc{\bX}{\mathbf{X}}
\nc{\be}{\mathbf{e}}
\nc{\bY}{\mathbf{Y}}
\nc{\bG}{\mathbf{G}}
\nc{\bz}{\mathbf{z}}
\nc{\bw}{\mathbf{w}}
\nc{\bA}{\mathbf{A}}
\nc{\bJ}{\mathbf{J}}
\nc{\bK}{\mathbf{K}}
\nc{\bb}{\mathbf{b}}
\nc{\ba}{\mathbf{a}}
\nc{\bs}{\mathbf{s}}
\nc{\bzero}{\mathbf{0}}
\nc{\bi}{\mathbf{i}}
\nc{\Edistinct}{\ME^{\mathsf{distinct}}}
\nc{\bc}{\mathbf{c}}
\nc{\bC}{\mathbf{C}}
\nc{\BR}{\mathbb R}
\nc{\BA}{\mathbb{A}}
\nc{\SA}{\mathscr{A}}
\nc{\BC}{\mathbb C}
\nc{\bx}{\mathbf{x}}
\nc{\bS}{\mathbf{S}}
\nc{\bM}{\mathbf{M}}
\nc{\bR}{\mathbf{R}}
\nc{\bN}{\mathbf{N}}
\nc{\by}{\mathbf{y}}
\nc{\sy}{y}
\nc{\sx}{x}

\nc{\cF}{\mathcal{F}}
\nc{\cE}{\mathcal{E}}
\nc{\MO}{\mathcal O}
\nc{\MQ}{\mathcal{Q}}
\nc{\CO}{\mathscr{O}}
\nc{\MU}{\mathcal{U}}
\nc{\ME}{\mathcal{E}}
\nc{\MN}{\mathcal{N}}
\nc{\MK}{\mathcal{K}}
\nc{\MM}{\mathcal{M}}
\nc{\MS}{\mathcal{S}}
\nc{\MT}{\mathcal{T}}
\nc{\BF}{\mathbb F}
\nc{\BQ}{\mathbb Q}
\nc{\MX}{\mathcal{X}}
\nc{\MA}{\mathcal{A}}
\nc{\MD}{\mathcal{D}}
\nc{\MB}{\mathcal{B}}
\nc{\MZ}{\mathcal{Z}}
\nc{\MJ}{\mathcal{J}}
\nc{\MW}{\mathcal{W}}
\nc{\MR}{\mathcal{R}}
\nc{\MY}{\mathcal{Y}}
\nc{\ML}{\mathcal{L}}
\nc{\BZ}{\mathbb Z}
\nc{\BN}{\mathbb N}
\nc{\ep}{\epsilon}
\nc{\gapfn}[1]{\varepsilon_{#1}}
\nc{\ggapfn}[2]{\varphi_{#1}({#2})}
\nc{\epsahk}{\gapfn{0}}
\nc{\BH}{\mathbb H}
\nc{\BG}{\mathbb{G}}
\nc{\D}{\Delta}
\nc{\MF}{\mathcal{F}}
\nc{\One}{\mathbbm{1}}
\nc{\bOne}{\mathbf{1}}
\nc{\Aopt}{\mathcal{A}^{\rm opt}}
\nc{\Amul}{\mathcal{A}^{\rm mul}}

\nc{\SP}{\mathsf P}
\nc{\SQ}{\mathsf Q}

\nc{\DO}{\accentset{\circ}{\D}}
\nc{\mf}{\mathfrak}
\nc{\mfp}{\mathfrak{p}}
\nc{\mfq}{\mf{q}}
\nc{\Sp}{\mbox{Spec}}
\nc{\Spm}{\mbox{Specm}}
\nc{\hookuparrow}{\mathrel{\rotatebox[origin=c]{90}{$\hookrightarrow$}}}
\nc{\hookdownarrow}{\mathrel{\rotatebox[origin=c]{-90}{$\hookrightarrow$}}}
\nc{\hra}{\hookrightarrow}
\nc{\tra}{\twoheadrightarrow}
\nc{\sgn}{{\rm sgn}}
\nc{\aut}{{\rm Aut}}
\nc{\Hom}{{\rm Hom}}
\nc{\img}{{\rm Im}}
\DMO{\id}{Id}
\DMO{\supp}{supp}
\DMO{\KL}{KL}
\nc{\kld}[2]{\KL({#1}||{#2})}
\nc{\ren}[2]{D_2({#1}||{#2})}
\nc{\chisq}[2]{\chi^2({#1}||{#2})}
\nc{\tvd}[2]{D_{\mathsf{TV}}\left({#1}, {#2}\right)}
\nc{\hell}[2]{H^2({#1}, {#2})}
\DMO{\BSS}{BSS}
\DMO{\BES}{BES}
\DMO{\BGS}{BGS}
\DMO{\poly}{poly}
\nc{\indep}{\perp}
\DMO{\sink}{sink}
\DMO{\nosink}{nosink}
\nc{\sinks}{s^{\sink}}
\nc{\sinkobs}{o^{\sink}}
\nc{\fp}[1]{\MP_1({#1})}
\nc{\BO}{\mathbb{O}}
\nc{\BT}{\mathbb{T}}

\nc{\RR}{\mathbb{R}}
\nc{\NN}{\mathbb{N}}
\nc{\Gradient}{\nabla}
\DMO{\diag}{diag}
\nc{\norm}[1]{\left \lVert #1 \right \rVert}
\DMO*{\EE}{\mathbb{E}}
\nc{\LPN}{\mathsf{LPN}}
\DMO{\Ber}{Ber}
\nc{\Regress}{\mathtt{Regress}}
\nc{\LFC}{\mathtt{LearnFromCorr}}
\nc{\RegressAlg}{\mathtt{RegressAlg}}
\nc{\DrawTraj}{\mathtt{DrawTrajectory}}
\nc{\pizero}{{\pi_{\mathsf{zero}}}}
\nc{\Tred}{{T_{\mathsf{red}}}}
\nc{\epred}{{\epsilon_{\mathsf{red}}}}
\nc{\TriAlg}{\mathtt{GenerateTriangleLPN}}
\nc{\Alg}{\mathtt{Alg}}
\nc{\AffSample}{\mathtt{AffSample}}

\nc{\br}{\mathbf{r}}
\nc{\TV}{{\mathsf{TV}}}
\DMO{\Law}{Law}
\DMO{\Sym}{Sym}
\nc{\bu}{\mathbf{u}}
\nc{\Reg}{\mathtt{Reg}}
\nc{\Breg}{B_{\mathsf{Reg}}}
\DMO{\dc}{dc}

\DMO{\PR}{Pr}
\renewcommand{\Pr}{\PR}
\DMO*{\Prr}{Pr}
\nc{\E}{\mathbb{E}}
\nc{\ra}{\rightarrow}

\nc{\fq}{\mathfrak{q}}
\nc{\laplace}{\Delta}
\nc{\pavg}{\overline p}
\nc{\CPI}{C_{\mathsf{PI}}}
\nc{\convolve}{\star}
\nc{\normal}{\mathcal{N}}
\nc{\phat}{\widehat p}
\nc{\tstar}{t^\star}
\nc{\dd}{\mathrm{d}}
\nc{\nutil}{\widetilde \nu}

\nc{\pstar}{p^\star}
\nc{\ind}[1]{^{\footnotesize (#1)}}
\nc{\cL}{\mathcal{L}}
\nc{\qtil}{\widetilde{q}}
\nc{\indic}{\mathbbm{1}}
\nc{\pbar}{\ol p}
\nc{\qbar}{\ol q}
\nc{\qstar}{q^\star}
\nc{\epsscore}{\epsilon_{\mathsf{score}}}

\nc{\khat}{\hat \kappa}
\nc{\Zhat}{\hat Z}

\newcommand{\op}{\mathsf{op}}
\newcommand{\Lip}{\mathsf{Lip}}

\nc{\epfinal}{\epsilon_{\mathsf{final}}}
\nc{\cS}{\mathcal{S}}
\nc{\xtil}{\widetilde x}
\nc{\ztil}{\widetilde z}
\nc{\twist}[1]{p_{#1}}
\nc{\htwist}[1]{\phat_{#1}}
\nc{\Cnorm}{C_{\mathsf{norm}}}
\nc{\Clip}{C_{\mathsf{lip}}}
\nc{\Ckl}{C_{\mathsf{KL}}}
\nc{\WW}{\mathsf{W}}
\DMO{\rank}{rank}
\nc{\rr}{\mathbf{r}}
\nc{\NP}{\mathrm{NP}}
\nc{\BPP}{\mathrm{BPP}}
\crefname{appendix}{Appendix}{Appendices}

\AtBeginEnvironment{appendices}{%
  \crefalias{section}{appendix}%
}

\crefname{assumption}{Assumption}{Assumptions}

\title{The tractability landscape of diffusion alignment: regularization, rewards, and computational primitives}

\ifneurips
\author{%
  David S.~Hippocampus\thanks{Use footnote for providing further information
    about author (webpage, alternative address)---\emph{not} for acknowledging
    funding agencies.} \\
  Department of Computer Science\\
  Cranberry-Lemon University\\
  Pittsburgh, PA 15213 \\
  \texttt{hippo@cs.cranberry-lemon.edu} \\
}
\fi 

\ifarxiv 
\author{}
\date{}
\renewcommand{\citet}[1]{\cite{#1}}
\renewcommand{\citep}[1]{\cite{#1}}
\fi

\begin{document}

\maketitle

\ifarxiv 
\vspace{-3em}
\begin{center}
\large
\setlength{\tabcolsep}{20pt}
\begin{tabular}{ccc}
\makecell{Ankur Moitra \\ \small{\texttt{moitra@mit.edu}}}
&
\makecell{Andrej Risteski \\ \small{\texttt{aristesk@andrew.cmu.edu}}}
&
\makecell{Dhruv Rohatgi \\ \small{\texttt{drohatgi@mit.edu}}}
\end{tabular}
\end{center}
\vspace{1em}
\fi

\begin{abstract}
Inference-time reward alignment asks how to turn a pre-trained diffusion model with base law $p$ into a sampler that favors a reward $\rr$ while remaining close to $p$. Since there is no canonical distributional distance for this closeness constraint, different choices lead to different ``reward-aligned'' laws and, just as importantly, different algorithmic problems. We develop a primitive-based approach to reward alignment: rather than assuming arbitrary reward-aligned laws can be sampled, we ask which simple algorithmic primitives suffice to implement alignment for non-trivial reward classes. If closeness is measured in KL distance, the target law is $q(x) \propto p(x) \exp(\lambda^{-1}\rr(x))$. For this setting, we show that linear exponential tilts of the form $q(x)\propto p(x)\exp(\langle \theta, x \rangle)$--- which according to recent work \citep{MoitraRisteskiRohatgi2026} can be efficiently sampled from --- are a sufficient primitive for aligning to a very broad class of convex low-dimensional rewards. If closeness is measured in Wasserstein distance, the corresponding primitive is a proximal transport oracle: given $x$, solve $\mbox{argmax}_y \left\{ \rr(y)- \lambda c(x,y) \right\}$. This oracle can be efficiently implemented for concave or low-dimensional Lipschitz rewards \(\rr(x)=f(Ax)\). Together, these results illustrate that the choice of distribution distance for alignment affects the computational primitive and the tractable reward class.
\end{abstract}


\section{Introduction}

Inference-time algorithms use a pre-trained generative model as a subroutine inside a procedure for solving downstream tasks. With diffusion models, the procedure often involves \emph{steering} or \emph{aligning} the base model toward generations that satisfy a preference, constraint, measurement, or scientific design objective. Example tasks include posterior sampling in inverse problems such as inpainting and super-resolution~\citep{yu2018generative}, conditional or meta-generation~\citep{welleck2024metageneration}, and steering toward samples with desired scientific properties~\citep{geffner2025proteina}. These settings share a common qualitative goal: starting from a base distribution $p$ learned by a pre-trained diffusion model, and some reward function $\rr$, we want to modify the sampling procedure, at inference time, to favor high-reward samples---while remaining close to $p$.

The natural way to formalize this qualitative goal is to sample from the solution to a regularized variational objective:
\[
    \argmax_{q} \; \EE_{x\sim q}[\rr(x)] - \lambda D(q,p),
\]
where $D$ is a distributional distance or divergence. 
However, this formulation makes apparent a basic ambiguity: there is often no canonical choice of $D$. For example, if $D$ is chosen to be KL divergence, the optimizer is the exponentially tilted law
\[
    p^\star(x) \propto p(x)\exp(\lambda^{-1} \rr(x)).
\]
If $D$ is chosen to be a Wasserstein distance, the optimizer is instead described through a transport problem. These two choices are not simply different regularizers for the same computational task. They induce different distributions (see \cref{fig:kl-vs-wasserstein})---and the computational tractability of the resulting task could potentially change as well. Thus, in order to understand when reward alignment can be performed efficiently, it is not enough to specify the reward and the base model; one must also understand the geometry in which closeness to the base model is being enforced.

\begin{figure}[t]
    \centering
    \includegraphics[width=0.9\linewidth]{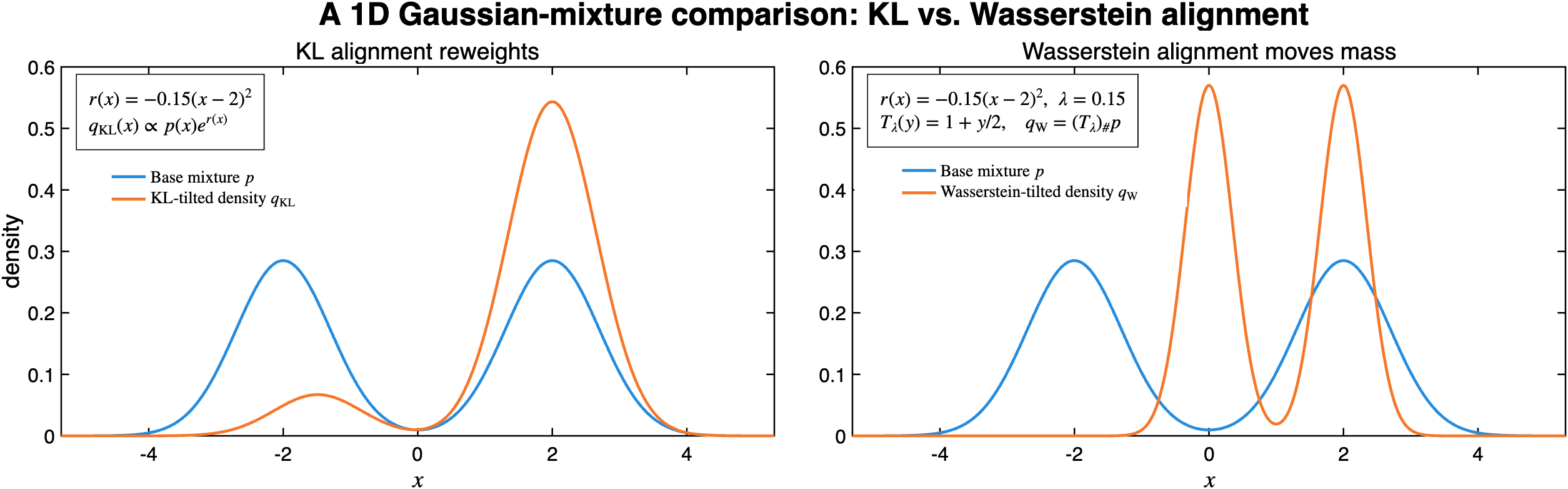}
    \caption{
    Illustration of KL versus Wasserstein alignment on a one-dimensional Gaussian mixture.
    The base distribution is $p=\frac12\mathcal N(-2,0.7^2)+\frac12\mathcal N(2,0.7^2)$ 
    and the reward is \(r(x)=-0.15(x-2)^2\).
    In the KL geometry, the aligned law
    \(q_{\mathrm{KL}}(x)\propto p(x)e^{r(x)}\)
    reweights the two modes.
    In the Wasserstein geometry, with quadratic transport cost and \(\lambda=0.15\), the proximal map is
    \(T_\lambda(y)=1+y/2\), so the aligned law
    \(q_{\mathrm W}=(T_\lambda)_\#p\)
    moves mass toward the high-reward region.
    }
    \label{fig:kl-vs-wasserstein}
\end{figure}

The main conceptual contribution of our paper is a primitive-based view of inference-time reward alignment. Namely, we identify elementary algorithmic primitives that allow us to characterize a broad class of rewards for which the aligned law can be provably, efficiently sampled from. 

\begin{itemize}[leftmargin=*]
\item \textbf{KL alignment using linear tilt oracles:}
We first study this question for KL alignment.     Our main positive result shows that if the reward is \emph{convex and low-dimensional}, namely
    \[
        \rr(x)=f(Ax), \qquad A\in\RR^{k\times d}, \qquad k\ll d,
    \]
with $f$ convex and sufficiently regular on the compact set containing $A\supp(p)$, then we can design an
    algorithm whose running time is exponential in the intrinsic dimension $k$ but polynomial in
    the ambient dimension $d$. Both assumptions (convexity and low-dimensionality) are provably necessary \citep{MoitraRisteskiRohatgi2026}.

The key idea is to build on the computational primitive of sampling from \emph{linear tilts} (\cref{def:linear-tilt}). This primitive can be implemented efficiently \citep{holderrieth2025glass,MoitraRisteskiRohatgi2026}, and it enables efficient KL alignment for any log-sum-exp reward function (with not too many summands):
\[\tilde{\rr}(x) = \log \sum_{i=1}^m w_i \exp(\langle v_i, x\rangle).\]
Unfortunately, it is unclear whether any convex $f$ can be approximated by a log-sum-exp function to arbitrary accuracy. However, we show how to construct a log-sum-exp function $G$ that is a coarse \emph{upper envelope} for $f$. The number of summands only scales with the intrinsic dimension $k$; thus, we can efficiently sample from $q(x) \propto p(x) \exp(G(Ax))$, and correct $q$ to the target distribution via rejection sampling. 

This result greatly generalizes the work of \citep{MoitraRisteskiRohatgi2026}, who handle the special case where $f$ is a positive-definite quadratic.

\item \textbf{Wasserstein alignment using proximal transport oracles}:     We then turn to Wasserstein alignment. In this geometry, the aligned distribution is not obtained by
    reweighting the base distribution. Instead, the solution to the variational alignment problem involves drawing $x\sim p$ and moving it
    to a nearby point $y$ with larger reward, while paying a transport cost for the displacement.
    The corresponding primitive is therefore a proximal transport oracle: 
    \[
        T_\lambda(x) \in \argmax_y \left\{\rr(y)-\lambda c(x,y)\right\}.
    \]
    Thus Wasserstein alignment reduces to the algorithmic problem of solving the reward-minus-cost
    maximization problem. For \emph{quadratic transport} costs, this is a proximal optimization problem,
    and it is efficiently solvable for \emph{concave} rewards $\rr$ or low-dimensional Lipschitz rewards \(\rr(x)=f(Ax)\) for $A\in\RR^{k\times d}$ with $k\ll d$.

\end{itemize}

Together, these results show that the choice of distributional distance changes the computational tractability landscape, as well as the relevant computational primitives. In KL geometry, alignment is a density-reweighting problem. Convex low-dimensional rewards
are tractable because their exponentials can be represented, via log-Laplace envelopes, as
positive mixtures of linear exponential tilts. In Wasserstein geometry, alignment is a transport problem. Concave rewards are tractable because reward-minus-quadratic-cost optimization becomes convex.
Low-dimensional rewards \(\rr(x)=f(Ax)\) lead to a different kind of tractability:
the Wasserstein proximal problem reduces to a global search over the low-dimensional coordinates
\(Ax\), and therefore can be solved in time exponential only in \(\rank(A)\); thus low-dimensionality
helps both geometries, but through different mechanisms.

This gives a particularly sharp contrast for concave rewards. Recent work by
\citet{MoitraRisteskiRohatgi2026} shows that KL alignment can be computationally intractable even
for a rank-one negative quadratic reward.\footnote{Specifically, we cannot sample from the tilted distribution even approximately unless $\mathrm{NP} \subseteq \mathrm{BPP}$---which is widely considered unlikely, as it would imply $\mathrm{NP}$-hard problems can be solved by randomized polynomial-time algorithms.} Under Wasserstein geometry, the same class of rewards is
tractable: the proximal transport problem is convex, and the corresponding aligned distribution can be sampled
by pushing forward base samples through the proximal map.

\subsection{Related work}

A large body of recent work studies inference-time steering of pretrained diffusion models, motivated by applications such as inverse problems, conditional generation, and preference optimization. Classifier guidance and related mechanisms were popularized as practical methods for conditional image generation~\citep{dhariwal2021diffusion}, and similar ideas have since been used for reward-guided sampling, posterior inference, and alignment. Several works have begun to analyze when steering can be provably performed via KL alignment: \citet{gupta2024diffusion} show worst-case intractability for diffusion posterior sampling; \citet{bruna2024provable} study posterior sampling with denoising oracles via tilted transport; \citet{chidambaram2024does} analyze guidance in simple mixture models and show that it need not sample from the intended tilt; \citet{karan2025reguidance} propose a wrapper for hard inverse problems; and \citet{parulekar2025efficient,xun2025posterior} give positive results under alternative approximation notions or structural assumptions. Our work is complementary to this line: rather than focusing only on posterior sampling or a specific guidance heuristic, we ask how the choice of distributional geometry changes which primitives are required for alignment and which reward classes are computationally tractable.

Closest to our work is
\citet{MoitraRisteskiRohatgi2026}, who give a fine-grained analysis of this problem for KL alignment with quadratic
rewards. They show that with linear rewards, the tilt can be sampled from efficiently using a diffusion oracle (an observation also implicit in the work of \cite{holderrieth2025glass}),
use the Hubbard--Stratonovich transform to handle low-rank positive-semidefinite quadratic rewards,
and prove hardness for rank-one negative-semidefinite quadratic rewards. Our KL results build on
their linear-tilt primitive, but show that it can be used much more broadly: convex low-dimensional
rewards can be handled by constructing finite log-Laplace envelopes and reducing the resulting
proposal distribution to a mixture of linear tilts.

Our Wasserstein results are related to the literature on Wasserstein distributionally robust
optimization and optimal-transport model risk. In this literature, worst-case expectations over
Wasserstein balls admit dual formulations involving a Lagrange multiplier and a pointwise
supremum of the form \(\sup_x\{\ell(x)-\lambda c(x,y)\}\)
\citep{BlanchetMurthy2019,GaoKleywegt2023,ZhangYangGao2025}. We use this same dual perspective,
but in a generative-model setting: the pointwise maximization is an algorithmic primitive for sampling from the Wasserstein tilt.

\section{Preliminaries}\label{sec:preliminaries}

\subsection{General preliminaries}

\paragraph{Basic notation.}
For a probability measure $P$ on $\RR^d$, we write $X\sim P$ for a random variable with law $P$.
When $P$ has a density with respect to an underlying base measure $\mu$, we denote the density by $p$. 
In the discrete case, $\mu$ should be interpreted as counting measure. 
For a measurable map
$T:\RR^d\to\RR^d$, we write $T_\#P$ for the pushforward of $P$ by $T$, i.e. the law of $T(X)$
when $X\sim P$. For $R>0$, let $\CB_{d,2}(R) := \{x\in\RR^d:\norm{x}_2\leq R\}.$ 
We write $\supp(P)$ for the support of $P$. Given $A\in\RR^{k\times d}$, define $A\supp(P) := \{Ax: x \in \supp(P)\}$.

For two probability measures $P,Q$, we define total variation distance as
\[
    \TV(P,Q)
    :=
    \sup_{S} |P(S)-Q(S)|
    =
    \frac12\int \left|p(x)-q(x)\right|\,\dd \mu(x),
\]
whenever the densities exist. We write $\WW_2(P,Q)$ for Wasserstein-$2$ distance:
\[
    \WW_2^2(P,Q)
    :=
    \inf_{\gamma\in\Pi(P,Q)}
    \EE_{(X,Y)\sim \gamma}\left[\norm{X-Y}_2^2\right],
\]
where $\Pi(P,Q)$ is the set of couplings of $P$ and $Q$.

\paragraph{Diffusion oracle access.}
Our base object is a distribution $P$ on $\RR^d$, which should be thought of as the law sampled by
a pre-trained diffusion model. Following the standard score-oracle abstraction, we assume access to
the scores of Gaussian noised versions of $P$.

\begin{definition}[Noised distribution]\label{def:noised-distribution}
For $\sigma\in[0,1]$, define $P_\sigma$ to be the law of
\[
    X_\sigma := \sqrt{1-\sigma^2}\,X+\sigma Z,
    \qquad
    X\sim P,\quad Z\sim \normal(0,I_d),
\]
with $X$ and $Z$ independent. We write $p_\sigma$ for the density of $P_\sigma$ when it exists.
\end{definition}

\begin{assumption}[Bounded support]\label{ass:bounded-support}
There is a known parameter $\Cnorm\geq 1$ such that $\supp(P)\subseteq \CB_{d,2}(\Cnorm).$
\end{assumption}

\begin{assumption}[Diffusion score oracle]\label{ass:score-oracle}
For every $\sigma\in(0,1)$ and every $x\in\RR^d$, we can query $s_\sigma(x) := \nabla \log p_\sigma(x).$
\end{assumption}

The bounded-support assumption ensures that all exponential tilts considered below are well-defined
for bounded reward parameters. The score-oracle assumption is the idealized form of access provided
by a pre-trained diffusion model---and widely assumed in prior related work \citep{ chen2023sampling,bruna2024provable,MoitraRisteskiRohatgi2026}. We do not investigate the effect of \emph{errors} in the score oracle; while they can be handled in the vanilla sampling setting (i.e. $\rr = 0$) \citep{chen2023sampling}, their effect on reward alignment remains largely an open and subtle question, as noted in \citep{MoitraRisteskiRohatgi2026}. 

\paragraph{Reward alignment.}
Given a reward function $\rr:\RR^d\to\RR$, the general reward alignment problem is to sample from a
distribution $Q$ that achieves high expected reward while remaining close to the base distribution
$P$. Formally, for a domain $\mathcal X \subseteq \RR^d$ and a divergence or distance $D$, we wish to sample from a solution to the following variational optimization problem:
\begin{equation}
    \argmax_{Q\in\Delta(\mathcal X)}
    \left\{
        \EE_{X\sim Q}[\rr(X)]
        -
        \lambda D(Q,P)
    \right\}.
    \label{eq:general-distance-tilt-problem}
\end{equation}

We study two choices for $D$: KL divergence and Wasserstein distance. 

\subsection{KL alignment and linear exponential tilts}

The following variational identity is standard.\footnote{For completeness, the proof is included in \cref{a:prelims}}

\begin{lemma}[KL alignment gives exponential weights]\label{lem:kl-exponential-tilt}
Let $\lambda>0$ and assume $Z_{\rr,\lambda}
    :=
    \EE_{X\sim P}\left[\exp(\rr(X)/\lambda)\right]
    <\infty.$
Then the unique optimizer of the KL alignment problem
\[
    \sup_{Q\ll P}
    \left\{
        \EE_{X\sim Q}[\rr(X)]-\lambda \KL(Q\|P)
    \right\}
\]
is the exponentially tilted distribution $P_{\rr,\lambda}$ defined by $P_{\rr,\lambda}(\dd x)
    :=
    \frac{\exp(\rr(x)/\lambda)}{Z_{\rr,\lambda}}P(\dd x).$
Equivalently, if $P$ has density $p$, then $p_{\rr,\lambda}(x)
    \propto
  p(x)\exp(\rr(x)/\lambda).$
\end{lemma}

The most basic example of an exponentially tilted distribution is a linear exponential tilt, which for conciseness we will also refer to as a \emph{linear tilt}.

\begin{definition}[Linear exponential tilt]\label{def:linear-tilt}
For $v\in\RR^d$, define \(
    Z_P(v)
    :=
    \EE_{X\sim P}\left[\exp(\langle v,X\rangle)\right],\)
and whenever $Z_P(v)<\infty$, define the linear exponential tilt of $P$ by $P_v(\dd x)
    :=
    \frac{\exp(\langle v,x\rangle)}{Z_P(v)}P(\dd x).$
Equivalently, if $P$ has density $p$, then $P_v$ has density $p(\cdot;v)$ defined by
$p(x;v)
    \propto
    p(x)\exp(\langle v,x\rangle).$
\end{definition}

Under \cref{ass:bounded-support}, $Z_P(v)<\infty$ for every $v\in\RR^d$. Our algorithmic results for KL tilting use the following primitive, which packages the linear-tilt
sampler and normalizer-estimation routines available from recent work by \citep{MoitraRisteskiRohatgi2026}.

\begin{definition}[Linear-tilt primitive, \citep{MoitraRisteskiRohatgi2026}]\label{def:linear-tilt-primitive}
We say that $P$ admits an efficient linear-tilt primitive if, for any $v\in\RR^d$, the following two
tasks can be performed in time polynomial in $d$, $\Cnorm$, $\norm{v}_2$, and the relevant accuracy
parameters:
\begin{enumerate}[leftmargin=*]
    \item \textbf{Sampling.} Given $\epsilon>0$, output a sample from a distribution
    $\widehat P_v$ satisfying $\supp(\widehat P_v) \subseteq \CB_{d,2}(\Cnorm)$ and $\WW_2(\widehat P_v,P_v)\leq \epsilon.$
    \item \textbf{Normalizer estimation.} Given $\eta,\delta\in(0,1)$, output
    $\widehat Z_P(v)$ such that, with probability at least $1-\delta$,
    \[
        \widehat Z_P(v)
        \in
        [(1-\eta)Z_P(v),(1+\eta)Z_P(v)].
    \]
\end{enumerate}
\end{definition}

For constructive approximation results, we will assume first-order access to the low-dimensional
function $f$.

\begin{definition}[First-order oracle for convex rewards]\label{def:first-order-oracle}
Let $K\subseteq\RR^k$ and let $f$ be convex on a neighborhood of $K$. A first-order oracle for
$f$ is a procedure which, given $u\in K$, returns
\[
    \left(f(u),g(u)\right),
    \qquad
    g(u)\in \partial f(u).
\]
If $f$ is differentiable, we take $g(u)=\nabla f(u)$.
\end{definition}

\subsection{Wasserstein alignment and proximal transport}

We next recall the transport geometry used for Wasserstein alignment. Let $c:\RR^d\times\RR^d\to
\RR_{\geq 0}$ be a transportation cost. For probability measures $P,Q$, define
\[
    W_c(Q,P)
    :=
    \inf_{\gamma\in\Pi(Q,P)}
    \EE_{(Y,X)\sim\gamma}\left[c(X,Y)\right].
\]
For the quadratic cost $c(x,y)=\norm{x-y}_2^2$, this is $\WW_2^2(Q,P)$.
As we will show later, it turns out that the following primitive is central to the Wasserstein alignment problem.

\begin{definition}[Proximal transport oracle]\label{def:proximal-transport-oracle}
Given a domain $\mathcal X$, reward $\rr$, cost $c$, and parameter $\lambda\geq 0$, a proximal transport map is any
measurable $T_\lambda:\RR^d \to \RR^d$ satisfying, for all $y \in \RR^d$,
\[
    T_\lambda(y)
    \in
    \argmax_{x\in\mathcal X}
    \left\{
        \rr(x)-\lambda c(x,y)
    \right\}.
\]
\end{definition}

\section{Low-dimensional convex KL alignment via linear exponential tilts}
\label{sec:lowrank-convex}

In this section, we prove that KL alignment with low-dimensional convex rewards can be solved efficiently, using
linear exponential tilting as a key primitive. Throughout the section, we set the inverse-temperature parameter in \cref{eq:general-distance-tilt-problem} to $\lambda=1$. Thus,
for a reward of the form $\rr(x)=f(Ax), A\in\RR^{k\times d}$, our target distribution is
\begin{equation}
\label{eq:convex-kl-target}
    Q_f(dx)
    :=
    \frac{p(x)\exp(f(Ax))}{Z_f}\,\mu(dx),
    \qquad
    Z_f
    :=
    \int p(x)\exp(f(Ax))\,\mu(dx).
\end{equation}
An inverse-temperature parameter in the KL objective can be absorbed into the reward $f$, so this
normalization is without loss of generality.
The computational primitive the algorithms will use is approximate sampling from linear exponential tilts---as formalized in \cref{def:linear-tilt-primitive}. 

The main result of this section is a sampler for $Q_f$ for any convex $f$, whose
runtime is polynomial in the ambient dimension $d$ and exponential only in the intrinsic dimension
$k$.

\begin{theorem}[Low-dimensional convex KL alignment from linear tilts]
\label{thm:lowrank-convex-main}
Assume \cref{ass:bounded-support,ass:score-oracle}. Let
\[
    \rr(x)=f(Ax),
    \qquad
    A\in\RR^{k\times d},
    \qquad
    R:=\norm{A}_{\op}\Cnorm.
\]
Suppose that $f$ is convex and $L$-Lipschitz on an open convex neighborhood of
$\CB_{k,2}(R)$.
Let $Q_f$ be the KL tilt defined in \eqref{eq:convex-kl-target}. Define
\[
    M:=\left(1+4LR\right)^k.
\]
Then for every $\epsilon\in(0,1)$ and $\delta\in(0,1)$, there is a randomized algorithm which,
given diffusion-score access to $P$ and first-order oracle access to $f$, runs in time at most
\[
    \operatorname{poly}\!\left(
        d,\,
        k,\,
        \Cnorm,\,
        L,\,
        \norm{A}_{\op},\,
        \epsilon^{-1},\,
        \log(1/\delta),\,
        M
    \right)
\]
and outputs a sample whose law $\widehat Q$ satisfies $\WW_2(\widehat Q,Q_f)\leq \epsilon$ with probability at least $1-\delta$.
\end{theorem}

The full algorithm is specified\footnote{We note the last fallback step is only used to make the runtime deterministic.} as \cref{alg:lowrank-convex-tilt}. We provide the main ingredients of the proof; the complete proof is deferred to \cref{a:kl}.

\begin{algorithm}[t]
\caption{\textsc{LowRankConvexTiltSampler}: low-rank convex KL tilt}
\label{alg:lowrank-convex-tilt}
\begin{algorithmic}[1]
  \State \textbf{Input:} score oracle $(s_\sigma)_\sigma$ for $P$, matrix $A\in\RR^{k\times d}$, first-order oracle $\mathcal O_f$, parameters $L,R,\Cnorm,\epsilon,\delta$.
  \If{$L=0$}
    \State \Return $\textsc{LinTiltSampler}((s_\sigma)_\sigma,0,\epsilon,\Cnorm)$.
  \EndIf

  \State Set $h\gets 1/(2L)$ and construct an $h$-net $\{u^{(1)},\ldots,u^{(m)}\}\subseteq \CB_{k,2}(R)$, so $m\le (1+4LR)^k$.
  \For{$i=1,\ldots,m$}
    \State Query $(f_i,g_i)\gets \mathcal O_f(u^{(i)})$.
    \State Set $z_i\gets g_i$, $b_i\gets f_i-\langle g_i,u^{(i)}\rangle$, $w_i\gets e^{b_i}$, and $v_i\gets A^\top z_i$.
  \EndFor

  \State Define $G(u)\gets 1+\log\!\left(\sum_{i=1}^m w_i e^{\langle z_i,u\rangle}\right)$.
  \State Set $B\gets 1+\log m$, $a_0\gets e^{-B}$, and $L_a\gets 2L\norm{A}_{\op}$.
  \State Set
  \[
    \rho\gets
    \min\left\{
      \frac{\epsilon^2 a_0}{32\Cnorm(1+2\Cnorm L_a)},
      \frac{a_0}{4\max\{1,L_a\}}
    \right\},
    \quad
    \epsilon_{\rm lin}\gets \rho/2,
    \quad
    \eta\gets \min\left\{\frac12,\frac{\rho^2}{128\Cnorm^2}\right\}.
  \]
  \State Set $N_{\rm rej}\gets \left\lceil 2a_0^{-1}\log(16\Cnorm^2/\epsilon^2)\right\rceil$.

  \For{$i=1,\ldots,m$}
    \State Estimate $\widehat Z_i\gets
    \textsc{EstimateNormalization}((s_\sigma)_\sigma,v_i,\eta,\delta/m,\Cnorm)$.
  \EndFor
  \State Set $\widehat\pi_i\gets w_i\widehat Z_i/\sum_{j=1}^m w_j\widehat Z_j$ for each $i\in[m]$.

  \For{$t=1,\ldots,N_{\rm rej}$}
    \State Draw $\widehat I_t\sim \widehat\pi$ and
    $X_t\gets \textsc{LinTiltSampler}((s_\sigma)_\sigma,v_{\widehat I_t},\epsilon_{\rm lin},\Cnorm)$.
    \State Accept $X_t$ with probability $a(X_t):=\exp(f(AX_t)-G(AX_t))$.
    \If{$X_t$ is accepted}
      \State \Return $X_t$.
    \EndIf
  \EndFor

  \State \Return $\textsc{LinTiltSampler}((s_\sigma)_\sigma,0,\epsilon,\Cnorm)$.
\end{algorithmic}
\end{algorithm}


The first element is showing how to construct an upper envelope for convex $f$: 
\begin{lemma}[Log-sum-exp envelope for convex rewards]
\label{lem:lse-envelope-unit}
Let $f$ be convex and $L$-Lipschitz on an open convex neighborhood of $\CB_{k,2}(R)$. Let
$\{u^{(1)},\dots,u^{(m)}\}$ be an $h$-net of $\CB_{k,2}(R)$ with $h=\frac{1}{2L}.$ For each $i$, define function $\ell_i$ by
\[
    \ell_i(u)
    :=
    f(u^{(i)})+\langle g_i,u-u^{(i)}\rangle,
    \qquad
    g_i\in\partial f(u^{(i)}).
\]
Define
\[
    G(u):=
    1+\log\left(\sum_{i=1}^m \exp(\ell_i(u))\right).
\]
Then for all $u\in\CB_{k,2}(R)$,
\[
    f(u)\leq G(u)\leq f(u)+1+\log m.
\]
Moreover,
\[
    m\leq \left(1+4LR\right)^k.
\]
\end{lemma}

Next, we observe that for any log-sum-exp function $G$, the tilted distribution $Q_G$ (\cref{eq:convex-kl-target}) is an explicit mixture of linear tilts.

\begin{lemma}
\label{lem:proposal-mixture-unit}
Let $z_1,\dots,z_m\in\RR^d$ and $w_1,\dots,w_m\geq 0$. Let $G(u)
    :=
    1+\log\left(
        \sum_{i=1}^m w_i\exp(\langle z_i,u\rangle)
    \right).$
Define
\[
    Q_G(dx)
    :=
    \frac{p(x)\exp(G(Ax))}{Z_G}\,\mu(dx).
\]
For each $i$, define
\[
    v_i:=A^\top z_i,
    \qquad
    Z_i:=Z_P(v_i)
    =
    \EE_{X\sim P}\exp(\langle v_i,X\rangle).
\]
Then, we have 
\[
    Q_G
    =
    \sum_{i=1}^m \pi_i P_{v_i}, \quad \mbox{where} \quad
    \pi_i
    =
    \frac{w_iZ_i}{\sum_{j=1}^m w_jZ_j}.
\]
\end{lemma}

\begin{remark}
We note that the mixture decomposition can be viewed as a finite Laplace-transform
representation. Namely, if $\nu$ is a finite nonnegative measure on $\RR^k$ and
$e^{g(u)}=\int_{\RR^k} e^{\langle z,u\rangle}\,\nu(dz),$
then $g$ is a log-Laplace transform. In this case the KL tilt by $g(Ax)$ satisfies
\[
    p(x)e^{g(Ax)}=
    \int p(x)e^{\langle A^\top z,x\rangle}\,\nu(dz).
\]
Thus the nonlinear tilt is a mixture of linear exponential tilts. In particular, for a
discrete measure $\nu=\sum_{i=1}^m w_i\delta_{z_i}$,
\[
    e^{g(u)}=\sum_{i=1}^m w_i e^{\langle z_i,u\rangle},
\]
and the corresponding tilted distribution is an explicit finite mixture of the linear tilts
$P_{A^\top z_i}$, with mixture weights proportional to $w_i Z_P(A^\top z_i).$
\end{remark}


\paragraph{Proof overview for \cref{thm:lowrank-convex-main}.} For any log-sum-exp function $G$ as defined in \cref{lem:proposal-mixture-unit}, the lemma implies that we can efficiently (approximately) sample from $Q_G$ using the linear-tilt computational primitive (\cref{def:linear-tilt-primitive}): first, we estimate the normalization constants $Z_i$ and use these to sample a component $i \in [m]$ from the distribution $\pi$; then, we sample from the linear tilt $P_{v_i}$. Errors in both procedures can be handled by a standard perturbation analysis (\cref{lem:approx-envelope-proposal}).

\cref{lem:lse-envelope-unit} shows that for any convex and Lipschitz $f$, we can construct a log-sum-exp function $G$ such that $Q_G$ has bounded density ratio with respect to $Q_f$. Thus, the full algorithm proceeds by rejection sampling, using $Q_G$ as the proposal distribution. See \cref{a:kl-proof-thm} for the full proof. 

\begin{remark}[Hubbard--Stratonovich as a Gaussian Laplace transform] The paper \citep{MoitraRisteskiRohatgi2026}  handled the case of low-rank positive-definite quadratics---which is a special case of our results for $f$ a quadratic. In their paper they used the Hubbard-Stratonovich identity, which is related to our framework but uses a Gaussian mixing measure $\nu$: 
\[
    \EE_{Z\sim \normal(0,I_k)} e^{\langle u,Z\rangle}
    =
    e^{\|u\|_2^2/2}.
\]
Thus the positive semidefinite quadratic reward $f(u)=\|u\|_2^2/2$ is an exact
log-Laplace transform. Our convex-reward result replaces this Gaussian mixing measure by
a finite discrete measure, obtained from a log-sum-exp approximation to a general convex
function.
\end{remark}

\section{Wasserstein alignment via proximal transport}\label{sec:wasserstein}
We now consider reward alignment with Wasserstein regularization, focusing on $\WW_2^2$ for concreteness.
Let $\mathcal{X}\subseteq \mathbb{R}^d$ be compact and let $P$ be supported on $\mathcal{X}$. Recall $\WW_2^2(Q,P)=\inf_{\pi\in\Pi(Q,P)}\mathbb{E}_{(X,Y)\sim\pi}\|X-Y\|^2$. Our main tool is the following alternate characterization of the optimizer of the Wasserstein alignment problem: 

\begin{theorem}\label{cor:w2-map}

Fix $\lambda\ge 0$ and define a (possibly set-valued) map:

\[
    T_\lambda(y)\in
    \argmax_{x\in\mathcal X}
    \left\{r(x)-\lambda\|x-y\|^2\right\}.
\]
Then the law $Q_\lambda=(T_\lambda)_\#P$
satisfies
\[Q_\lambda \in \argmax_{Q\in\Delta(\mathcal X)}\left\{ \EE_Q[r] - \lambda\WW_2^2(Q,P)\right\}.\]
\end{theorem}

To prove the result, we interpret the above variational problem as optimizing over \emph{couplings} where the second variable has marginal distribution $P$, and observe that we can optimize over each conditional distribution (i.e. for fixed choice of second variable $y$) independently. See \cref{a:wasserstein} for details. \cref{cor:w2-map} implies that, for a fixed $\lambda\ge 0$, the relevant
algorithmic primitive for Wasserstein alignment is precisely a proximal transport oracle (\cref{def:proximal-transport-oracle}), instantiated with quadratic cost.

\begin{remark} If we have only an approximate oracle for $\widehat T_\lambda(y)$, we can straightforwardly use it to approximately sample from $Q_{\lambda}$. For a formal statement, see \cref{cor:concave-wasserstein-sampler}.
\end{remark}

Below, we discuss several cases when this primitive can be (approximately) implemented efficiently.

\subsection{Application: concave rewards}

First, when $r$ is concave, the above optimization primitive becomes a standard convex optimization problem. In particular, if $r$ has an
efficient first-order oracle and $\mathcal X$ admits efficient projection or separation, then
$T_\lambda(y)$ can be computed to arbitrary accuracy by standard convex optimization methods.  

\begin{remark} A close relationship to proximal maps can be seen as follows. If we define the convex function $ \phi(x):=-r(x)+\iota_{\mathcal X}(x),$
where $\iota_{\mathcal X}$ is the indicator function of $\mathcal X$, then
$   T_\lambda(y)
    =
    \operatorname{prox}_{\phi/(2\lambda)}(y),$
under the convention
\[
    \operatorname{prox}_{\gamma\phi}(y)
    :=
    \argmin_x
    \left\{
        \frac12\|x-y\|^2+\gamma\phi(x)
    \right\}.
\]
Thus Wasserstein alignment with concave rewards is implemented by a proximal map for the convex
function $-r$.
\end{remark}

\paragraph{Special case: quadratic rewards.} A negative definite quadratic is a special case of concave rewards, namely $\rr(x)=-x^\top Bx+b^\top x$ for $B\succeq 0.$ By the general machinery of this section it can be efficiently solved (and for some sets $\mathcal{X}$ it may even have a closed form solution). However, alignment to rewards of this kind is computationally intractable for KL tilting---even if $B$ is of rank 1. Precisely, \citep{MoitraRisteskiRohatgi2026} show that sampling from the corresponding tilt even approximately is impossible unless $\mathrm{NP} \subseteq \mathrm{BPP}$---which is widely considered unlikely in computational complexity, as it would imply $\mathrm{NP}$-hard problems can be solved by randomized polynomial-time algorithms.

We remark that, on the other hand, convex quadratic rewards are intractable for Wasserstein alignment (for convenience, we analyze the constrained program rather than the regularized program):

\begin{lemma}[Wasserstein alignment with convex rewards can be hard]
\label{lem:wasserstein-convex-sampling-hard}
Suppose there is a randomized polynomial-time algorithm with the following guarantee: given a
compact convex set \(\mathcal X\subseteq\RR^d\), a base distribution \(P\) supported on
\(\mathcal X\) with diffusion-score oracle access, and a convex quadratic reward \(r\), the
algorithm outputs a sample from a distribution \(\widehat Q\) satisfying $\WW_2(\widehat Q,Q^\star)\leq \frac14$ 
for some optimizer \(Q^\star\) of
\[
    \sup_{Q\in\Delta(\mathcal X):\,\WW_2(Q,P)\leq \sqrt d}
    \EE_{X\sim Q}[r(X)].
\]
Then, we have \(\mathrm{NP}\subseteq\mathrm{BPP}\).
\end{lemma}

The proof of this is provided  in \cref{a:wasserstein}, and proceeds by reduction from MAX-CUT.
\subsection{Application: low-rank rewards}

Finally, Wasserstein alignment can also leverage low-rank structure, regardless of whether $f$ is convex or concave. In this case, the optimization problem resulting from the proximal oracle can be efficiently implemented. Precisely, let us set $\mathcal X := \CB_{d,2}(\Cnorm)$, and define reward function $\rr(x)=f(Ax)$ for $A\in\RR^{k\times d}$. Let us define:
\[
    \OPT_\lambda(P)
    :=
    \sup_{Q\in\Delta(\mathcal X)}
    \left\{
        \EE_{X\sim Q}[f(AX)]
        -
        \lambda\WW_2^2(Q,P)
    \right\}  =
    \sup_{\gamma:\,Y\sim P}
    \EE_{(X,Y)\sim\gamma}
    \left[
        f(AX)-\lambda\|X-Y\|_2^2
    \right].
\] 
We will call a law \(Q\) \(\epsilon\)-optimal for this objective if
\[
    \EE_{X\sim Q}[f(AX)]-\lambda\WW_2^2(Q,P)
    \ge
    \OPT_\lambda(P)-\epsilon.
\]
We show that we can efficiently sample from an $\epsilon$-optimal distribution so long as $\rank(A)$ is not too large: 

\begin{theorem}[Low-rank Wasserstein alignment]
\label{thm:lowrank-wasserstein}
Assume \cref{ass:bounded-support,ass:score-oracle}.
Define $r_A:=\rank(A)$ and $S:=\|A\|_{\op}.$ Assume \(f\) is \(L\)-Lipschitz on \(\CB_{k,2}(S\Cnorm)\), and assume value-oracle access to \(f\),
where each oracle call has unit cost. Fix \(\lambda>0\) and $\epsilon,\delta \in (0,1)$. 

Define $M_W
    :=
    \left(
        1+\frac{2\Cnorm}{H_\epsilon}
    \right)^{r_A}$
where 
$H_\epsilon
    :=
    \min\left\{
        \frac{\epsilon}{6(LS+4\lambda\Cnorm)},
        \frac{\epsilon^2}{288\lambda^2\Cnorm^3}
    \right\}.$

Then for every \(\epsilon\in(0,1)\) and \(\delta\in(0,1)\), there is a randomized algorithm which,
given diffusion-score access to \(P\) and value-oracle access to \(f\), runs in time at most 
\[\poly\left(
        d,\,
        k,\,
        \Cnorm,\,
        L,\,
        S,\,
        \lambda,\,
        \epsilon^{-1},\,
        \log(1/\delta),\,
        M_W
    \right)\]
and outputs a sample from a law \(\widehat Q\in\Delta(\mathcal X)\) satisfying $\epsilon$-optimality with probability at least \(1-\delta\).

\end{theorem}
The proof of this is provided in \cref{a:wasserstein}; the basic idea is that the proximal oracle can be efficiently implemented in time exponential in $r_A$ (which is at most $k$) due to the low-dimensional structure of the problem and Lipschitzness of $f$. 

\section{Conclusion}

In this paper, we studied inference-time reward alignment through the lens of computational primitives. Our results
show that changing the distributional geometry (i.e. the choice of regularization in \cref{eq:general-distance-tilt-problem}) can dramatically alter the tractability landscape. KL alignment is fundamentally a density
reweighting problem: convex low-dimensional rewards become tractable through log-Laplace
envelopes and mixtures of linear exponential tilts. Wasserstein alignment is fundamentally a transport problem:
concave rewards become tractable because calculating the proximal map is a convex optimization problem, while low-dimensional Lipschitz rewards
can be handled by exhaustive search in the appropriate low-dimensional space. Thus the same structural
assumption can be useful or intractable in different ways under different geometries. A broader goal is to develop a taxonomy of distributional distances according to the computational primitives they induce, and to understand when a particular alignment task is the right one for a given application. It would also be interesting to study the analogous landscape for fine-tuning, where the choice of geometry may affect not only the target distribution, but also the optimization landscape of training.

\ifarxiv
\newpage
\paragraph{Acknowledgments.} AM is supported in part by a Microsoft Trustworthy AI Grant, NSF award CCF-2430381, ONR grant N00014-22-1-2339, and a David and Lucile Packard Fellowship. AR is supported in part by NSF awards IIS-2211907, CCF-2238523, IIS-2403275, an Amazon Research
Award, ONR award N000142512124, a Google Research Scholar Award, and an OpenAI Superalignment Fast Grant. DR is supported by NSF awards CCF-2430381 and DMS-2022448, and ONR grant N00014-22-1-2339. The authors also thank Nick Boffi and Stephen Huan for helpful discussions.
\fi

\ifneurips
\bibliographystyle{plainnat}
\fi 
\ifarxiv
\bibliographystyle{alpha}
\fi
\bibliography{bib}

\begin{appendices}

\section{Relegated preliminaries}
\label{a:prelims}

We will repeatedly use the following elementary conversion between total variation and Wasserstein
distance under bounded support.

\begin{lemma}\label{lem:tv-to-w2}
Suppose $P$ and $Q$ are supported on $\CB_{d,2}(C)$. Then
\[
    \WW_2(P,Q) \leq 2C\sqrt{\TV(P,Q)}.
\]
\end{lemma}

\begin{proof}
Let $(X,Y)$ be a maximal coupling of $P$ and $Q$, so that
$\Pr[X\neq Y]=\TV(P,Q)$. Since both distributions are supported on $\CB_{d,2}(C)$,
we have $\norm{X-Y}_2\leq 2C$ always. Therefore
\[
    \EE \norm{X-Y}_2^2
    \leq 4C^2 \Pr[X\neq Y]
    =
    4C^2 \TV(P,Q).
\]
Taking square roots gives the claim.
\end{proof}

\begin{theorem}[Sampling given a score oracle {\citep{chen2023sampling}}]
\label{thm:chen-base-sampling}
Suppose \(P\) satisfies \cref{ass:bounded-support,ass:score-oracle}.
Then, for every \(\epsilon_P>0\) and \(\delta\in(0,1)\), there is an algorithm
\[
    \textsc{BaseSampler}((s_\sigma)_\sigma,\epsilon_P,\delta,\Cnorm)
\]
which runs in time
\[
    \poly(d,\Cnorm,\epsilon_P^{-1},\log(1/\delta))
\]
and outputs a sample from a distribution \(\widetilde P\) satisfying
\[
   \WW_2(\widetilde P,P)\le \epsilon_P
\]
with probability at least \(1-\delta\). Moreover, after projecting the output onto
\(\CB_{d,2}(\Cnorm)\), the same guarantee holds and the output distribution is supported on
\(\CB_{d,2}(\Cnorm)\).
\end{theorem}

\begin{proof}[Proof of Lemma~\ref{lem:kl-exponential-tilt}]
Let $\widetilde P$ denote the distribution with density
$\widetilde p(x)=p(x)\exp(\rr(x)/\lambda)/Z_{\rr,\lambda}$. Then
\[
    \KL(Q\|\widetilde P)
    =
    \KL(Q\|P)
    -
    \frac{1}{\lambda}\EE_Q[\rr(X)]
    +
    \log Z_{\rr,\lambda}.
\]
Rearranging gives
\[
    \EE_Q[\rr(X)]-\lambda\KL(Q\|P)
    =
    \lambda\log Z_{\rr,\lambda}
    -
    \lambda\KL(Q\|\widetilde P).
\]
The right-hand side is maximized uniquely when $Q=\widetilde P$.
\end{proof}

\section{Omitted proofs and technical lemmas for \cref{sec:lowrank-convex}}\label{a:kl}

In \cref{a:kl-lemma-proofs} we prove \cref{lem:lse-envelope-unit,lem:proposal-mixture-unit}, the main ingredients in the proof of \cref{thm:lowrank-convex-main}. In \cref{a:kl-proof-thm} we then prove \cref{thm:lowrank-convex-main}. \cref{a:kl-technical-lemmas} contains additional technical lemmas needed for the proof.

\subsection{Omitted proofs of main lemmas}\label{a:kl-lemma-proofs}

\begin{proof}[Proof of \cref{lem:lse-envelope-unit}]
Let $\phi(u):=\max_{i\in[m]}\ell_i(u).$ 
By convexity, each $\ell_i$ is a supporting hyperplane, so $\ell_i(u)\leq f(u),  \forall u.$ Thus $\phi(u)\leq f(u)$. Fix $u\in\CB_{k,2}(R)$ and choose $i$ such that $\norm{u-u^{(i)}}_2\leq h.$ 
Since $f$ is $L$-Lipschitz, \cref{lem:subgrad-bounded-unit} gives $\norm{g_i}_2\leq L.$ 
Hence, $\ell_i(u)
    \geq
    f(u^{(i)})-Lh.$
By Lipschitzness, we also have $f(u)\leq f(u^{(i)})+Lh.$
Combining the two inequalities gives:
\[
    f(u)-\phi(u)
    \leq
    f(u)-\ell_i(u)
    \leq
    2Lh
    =
    1.
\]
Therefore
\[
    f(u)-1\leq \phi(u)\leq f(u).
\]

The softmax bound gives
\[
    \phi(u)
    \leq
    \log\left(\sum_{i=1}^m \exp(\ell_i(u))\right)
    \leq
    \phi(u)+\log m.
\]
From this, it directly follows that
\[
    G(u)\geq 1+\phi(u)\geq f(u),
\]
and
\[
    G(u)\leq 1+\phi(u)+\log m\leq f(u)+1+\log m.
\]
Finally, the covering bound for an $h$-net of $\CB_{k,2}(R)$ gives
\[
    m\leq \left(1+\frac{2R}{h}\right)^k
    =
    \left(1+4LR\right)^k.
\]
\end{proof}

\begin{proof}[Proof of \cref{lem:proposal-mixture-unit}]
By definition,
\[
    \exp(G(Ax))
    =
    e
    \sum_{i=1}^m
    w_i\exp(\langle z_i,Ax\rangle)
    =
    e
    \sum_{i=1}^m
    w_i\exp(\langle A^\top z_i,x\rangle).
\]
Since $v_i=A^\top z_i$, we get
\[
    \exp(G(Ax))
    =
    e
    \sum_{i=1}^m
    w_i\exp(\langle v_i,x\rangle).
\]
Therefore
\[
    Q_G(dx)
    \propto
    \sum_{i=1}^m
    w_i\exp(\langle v_i,x\rangle)P(dx)
    =
    \sum_{i=1}^m
    w_iZ_iP_{v_i}(dx).
\]
Normalizing gives the stated mixture weights.
\end{proof}

\subsection{Proof of \cref{thm:lowrank-convex-main}}\label{a:kl-proof-thm}

\begin{proof}[Proof of \cref{thm:lowrank-convex-main}]
Run \textsc{LowRankConvexTiltSampler} (\cref{alg:lowrank-convex-tilt}). By \cref{lem:lse-envelope-unit}, the algorithm
constructs an envelope $G$ satisfying
\[
    f(u)\leq G(u)\leq f(u)+B,
    \qquad
    B:=1+\log m,
\]
for all $u\in\CB_{k,2}(R)$, with
\[
    m\leq M:=\left(1+4LR\right)^k.
\]
Since $A\supp(P)\subseteq\CB_{k,2}(R)$, the same bound holds for $u=Ax$ whenever
$x\in\supp(P)$.

By \cref{lem:proposal-mixture-unit}, the envelope proposal
\[
    Q_G(dx)\propto p(x)\exp(G(Ax))\,\mu(dx)
\]
is a finite mixture of linear tilts. The algorithm samples from an approximate version
$\widehat Q_G$ of this proposal by using approximate normalizer estimates and the approximate
linear-tilt sampler of~\citep{MoitraRisteskiRohatgi2026}.

With probability at least $1-\delta$, all $m$ normalizer estimates are simultaneously
multiplicatively accurate, by the union bound. Condition on this event occuring. By
\cref{lem:approx-envelope-proposal} and the choices
\[
    \epsilon_{\mathrm{lin}}=\rho/2,
    \qquad
    \eta
    =
    \min\left\{
        \frac12,\,
        \frac{\rho^2}{128\Cnorm^2}
    \right\},
\]
we have
\[
    \WW_2(\widehat Q_G,Q_G)\leq \rho.
\]

Now define
\[
    a(x):=\exp(f(Ax)-G(Ax)).
\]
Since $f\leq G\leq f+B$ on $A\supp(P)$,
\[
    e^{-B}\leq a(x)\leq 1
    \qquad
    \forall x\in\supp(P).
\]
Also, $f(Ax)$ is $L\norm{A}_{\op}$-Lipschitz in $x$. The function $G$ is
$L$-Lipschitz on $\CB_{k,2}(R)$ because it is a log-sum-exp of affine functions with slopes
$z_i=g_i$, and \cref{lem:subgrad-bounded-unit} gives $\norm{z_i}_2\leq L$. Therefore
$G(Ax)$ is also $L\norm{A}_{\op}$-Lipschitz in $x$. Hence $f(Ax)-G(Ax)$ is
$2L\norm{A}_{\op}$-Lipschitz. Since $a(x)\leq 1$, the acceptance function $a$ is $L_a:=2L\norm{A}_{\op}$-Lipschitz.

The exact correction of $Q_G$ by $a$ is the target $Q_f$, because
\[
    Q_G(dx)a(x)
    \propto
    p(x)\exp(G(Ax))\exp(f(Ax)-G(Ax))\,\mu(dx)
    =
    p(x)\exp(f(Ax))\,\mu(dx).
\]
The algorithm applies the same correction to $\widehat Q_G$. Let
\[
    \widehat Q_{\mathrm{acc}}:=\mathcal R_a(\widehat Q_G)
\]
be the law of a proposal conditioned on being accepted. By \cref{lem:rejection-stability-unit},
\[
    \WW_2(\widehat Q_{\mathrm{acc}},Q_f)
    \leq
    \sqrt{
        \frac{2\Cnorm(1+2\Cnorm L_a)}{a_0}
        \WW_2(\widehat Q_G,Q_G)
    }.
\]
Using $\WW_2(\widehat Q_G,Q_G)\leq \rho$ and the definition of $\rho$, we get
\[
    \WW_2(\widehat Q_{\mathrm{acc}},Q_f)
    \leq
    \epsilon/2.
\]

It remains to account for the fact that the algorithm runs a fixed number of rejection trials.
The exact proposal $Q_G$ has acceptance probability
\[
    \EE_{Q_G}[a(X)]
    =
    \frac{Z_f}{Z_G}
    \geq
    e^{-B}
    =
    a_0.
\]
For the approximate proposal, since $a$ is $L_a$-Lipschitz,
\[
    \left|
        \EE_{\widehat Q_G}[a(X)]-\EE_{Q_G}[a(X)]
    \right|
    \leq
    L_a \WW_1(\widehat Q_G,Q_G)
    \leq
    L_a \WW_2(\widehat Q_G,Q_G)
    \leq
    L_a\rho.
\]
By the choice of $\rho$, this is at most $a_0/4$. Therefore
\[
    \EE_{\widehat Q_G}[a(X)]
    \geq
    a_0/2.
\]
Since the algorithm runs
\[
    N_{\mathrm{rej}}
    =
    \left\lceil
        \frac{2}{a_0}
        \log\left(\frac{16\Cnorm^2}{\epsilon^2}\right)
    \right\rceil
\]
independent rejection trials, the probability of no acceptance is at most
\[
    \left(1-\frac{a_0}{2}\right)^{N_{\mathrm{rej}}}
    \leq
    \frac{\epsilon^2}{16\Cnorm^2}.
\]
If the fallback distribution after rejection failure is denoted by $H$, then the actual output law is
\[
    \widehat Q
    =
    (1-\gamma)\widehat Q_{\mathrm{acc}}+\gamma H,
    \qquad
    \gamma\leq \frac{\epsilon^2}{16\Cnorm^2}.
\]
Both $\widehat Q_{\mathrm{acc}}$ and $H$ are supported on $\CB_{d,2}(\Cnorm)$, so
\cref{lem:tv-to-w2} gives
\[
    \WW_2(\widehat Q,\widehat Q_{\mathrm{acc}})
    \leq
    2\Cnorm\sqrt{\gamma}
    \leq
    \epsilon/2.
\]
Thus, by the triangle inequality,
\[
    \WW_2(\widehat Q,Q_f)
    \leq
    \WW_2(\widehat Q,\widehat Q_{\mathrm{acc}})
    +
    \WW_2(\widehat Q_{\mathrm{acc}},Q_f)
    \leq
    \epsilon.
\]

Finally, we bound the runtime. The number of envelope pieces satisfies $m\leq M$. Constructing
the envelope requires $m$ first-order oracle calls and polynomial time in $m$ and $k$. Each tilt
vector satisfies
\[
    \norm{v_i}_2
    =
    \norm{A^\top z_i}_2
    \leq
    \norm{A}_{\op}L.
\]
The normalizer-estimation routine is called $m$ times with accuracy parameter $\eta$ and failure
probability $\delta/m$. The approximate linear-tilt sampler is called at most
$N_{\mathrm{rej}}$ times, each with accuracy $\epsilon_{\mathrm{lin}}$, hence the runtime follows from~\citep{MoitraRisteskiRohatgi2026}.
\end{proof}

\subsection{Technical lemmas}\label{a:kl-technical-lemmas}

The next lemma controls the error from approximate normalizer estimates and approximate linear-tilt
samples.

\begin{lemma}[Approximate sampling from the envelope proposal]
\label{lem:approx-envelope-proposal}
Let $Q_G=\sum_{i=1}^m\pi_iP_{v_i}$ 
be the envelope proposal from \cref{lem:proposal-mixture-unit}. Suppose the normalizer
estimates satisfy
\[
    \widehat Z_i\in[(1-\eta)Z_i,(1+\eta)Z_i]
    \qquad
    \forall i\in[m],
\]
and define
\[
    \widehat\pi_i
    :=
    \frac{w_i\widehat Z_i}{\sum_{j=1}^m w_j\widehat Z_j}.
\]
Suppose further that, for every $i$, the approximate linear-tilt sampler returns a law 
$\widehat P_{v_i}$ satisfying
\[
    \WW_2(\widehat P_{v_i},P_{v_i})\leq \epsilon_{\mathrm{lin}},
    \qquad
    \supp(\widehat P_{v_i})\subseteq \CB_{d,2}(\Cnorm).
\]
Let $\widehat Q_G
    :=
    \sum_{i=1}^m \widehat\pi_i\widehat P_{v_i}.$ Then
\[
    \WW_2(\widehat Q_G,Q_G)
    \leq
    \epsilon_{\mathrm{lin}}
    +
    2\Cnorm\sqrt{\frac{\eta}{1-\eta}}.
\]
\end{lemma}

\begin{proof}
Apply \cref{lem:weight-stability-unit} to the weights $a_i=w_iZ_i$ and  
    $\widehat a_i=w_i\widehat Z_i.$
This gives $\TV(\widehat\pi,\pi)
    \leq
    \frac{\eta}{1-\eta}.$
The claim then follows from \cref{lem:mixture-error-unit}.
\end{proof}

\begin{lemma}[Lipschitz convex functions have bounded subgradients]
\label{lem:subgrad-bounded-unit}
Let $U\subseteq\RR^k$ be open and convex, and let $f:U\to\RR$ be convex and $L$-Lipschitz on
$U$. Then for every $u\in U$ and every $g\in\partial f(u)$,
\[
    \norm{g}_2\leq L.
\]
\end{lemma}

\begin{proof}
Fix $u\in U$ and $g\in\partial f(u)$. If $g=0$, the result is immediate. Otherwise, for $t>0$
small enough,
\[
    u+t\frac{g}{\norm{g}_2}\in U.
\]
By the subgradient inequality,
\[
    f\left(u+t\frac{g}{\norm{g}_2}\right)
    \geq
    f(u)+t\norm{g}_2.
\]
By $L$-Lipschitzness,
\[
    f\left(u+t\frac{g}{\norm{g}_2}\right)
    \leq
    f(u)+Lt.
\]
Thus $\norm{g}_2\leq L$.
\end{proof}

\begin{lemma}[Softmax bounds]
\label{lem:softmax-unit}
For $a_1,\dots,a_m\in\RR$,
\[
    \max_i a_i
    \leq
    \log\left(\sum_{i=1}^m e^{a_i}\right)
    \leq
    \max_i a_i+\log m.
\]
\end{lemma}

\begin{proof}
Let $M=\max_i a_i$. Then
\[
    e^M
    \leq
    \sum_{i=1}^m e^{a_i}
    \leq
    me^M.
\]
Taking logarithms gives the result.
\end{proof}

\begin{lemma}[Stability of normalized weights]
\label{lem:weight-stability-unit}
Let $a_i>0$ and $\widehat a_i>0$ satisfy
\[
    \widehat a_i\in[(1-\eta)a_i,(1+\eta)a_i]
    \qquad
    \forall i\in[m],
\]
with $\eta\in(0,1)$. Let
\[
    \pi_i:=\frac{a_i}{\sum_j a_j},
    \qquad
    \widehat\pi_i:=\frac{\widehat a_i}{\sum_j \widehat a_j}.
\]
Then
\[
    \TV(\widehat\pi,\pi)
    \leq
    \frac{\eta}{1-\eta}.
\]
\end{lemma}

\begin{proof}
Let $A=\sum_j a_j$ and $\widehat A=\sum_j \widehat a_j$. Then
\[
    (1-\eta)A\leq \widehat A\leq (1+\eta)A.
\]
For each $i$,
\[
    |\widehat\pi_i-\pi_i|
    =
    \left|
        \frac{\widehat a_i}{\widehat A}
        -
        \frac{a_i}{A}
    \right|
    =
    \frac{|\widehat a_iA-a_i\widehat A|}{A\widehat A}.
\]
Using
\[
    |\widehat a_i-a_i|\leq \eta a_i,
    \qquad
    |\widehat A-A|\leq \eta A,
\]
we get
\[
    |\widehat a_iA-a_i\widehat A|
    \leq
    2\eta a_iA.
\]
Thus
\[
    |\widehat\pi_i-\pi_i|
    \leq
    \frac{2\eta a_i}{(1-\eta)A}.
\]
Summing over $i$ and dividing by $2$ gives the total-variation bound.
\end{proof}

\begin{lemma}[Mixture sampling error]
\label{lem:mixture-error-unit}
Assume all component distributions are supported on $\CB_{d,2}(\Cnorm)$. Let
\[
    Q=\sum_{i=1}^m \pi_iQ_i,
    \qquad
    \widehat Q=\sum_{i=1}^m \widehat\pi_i\widehat Q_i.
\]
Suppose
\[
    \TV(\widehat\pi,\pi)\leq \alpha
\]
and
\[
    \WW_2(\widehat Q_i,Q_i)\leq \epsilon
    \qquad
    \forall i\in[m].
\]
Then
\[
    \WW_2(\widehat Q,Q)
    \leq
    \epsilon+2\Cnorm\sqrt{\alpha}.
\]
\end{lemma}

\begin{proof}
Couple the mixture indices $(I,\widehat I)$ by a maximal coupling, so that
\[
    \Pr[I\neq \widehat I]=\TV(\pi,\widehat\pi)\leq \alpha.
\]
For each $i$, couple $X_i\sim Q_i$ and $\widehat X_i\sim\widehat Q_i$ so that
\[
    \EE\norm{X_i-\widehat X_i}_2^2\leq \epsilon^2.
\]
Set
\[
    X:=X_I,
    \qquad
    \widehat X:=\widehat X_{\widehat I}.
\]
On the event $I=\widehat I$, the squared distance has expectation at most $\epsilon^2$. On the
event $I\neq \widehat I$, both variables lie in $\CB_{d,2}(\Cnorm)$, so
\[
    \norm{X-\widehat X}_2^2\leq 4\Cnorm^2.
\]
Therefore
\[
    \EE\norm{X-\widehat X}_2^2
    \leq
    \epsilon^2+4\Cnorm^2\alpha.
\]
Taking square roots gives the claim.
\end{proof}

\begin{lemma}[Bounded support converts $W_1$ to $W_2$]
\label{lem:w1-to-w2-unit}
Suppose $P$ and $Q$ are supported on $\CB_{d,2}(\Cnorm)$. Then
\[
    \WW_2(P,Q)
    \leq
    \sqrt{2\Cnorm\,\WW_1(P,Q)}.
\]
\end{lemma}

\begin{proof}
For any coupling $(X,Y)$ of $P$ and $Q$, since both variables lie in $\CB_{d,2}(\Cnorm)$,
\[
    \norm{X-Y}_2\leq 2\Cnorm.
\]
Hence
\[
    \norm{X-Y}_2^2
    \leq
    2\Cnorm\norm{X-Y}_2.
\]
Taking expectations and then the infimum over couplings gives
\[
    \WW_2^2(P,Q)
    \leq
    2\Cnorm\,\WW_1(P,Q).
\]
Taking square roots proves the claim.
\end{proof}

\begin{lemma}[Stability of the correction step]
\label{lem:rejection-stability-unit}
Let $Q$ and $\widehat Q$ be probability measures supported on $\CB_{d,2}(\Cnorm)$. Let
$a:\CB_{d,2}(\Cnorm)\to[0,1]$ satisfy
\[
    a(x)\geq a_0>0
    \qquad
    \forall x\in\CB_{d,2}(\Cnorm),
\]
and suppose $a$ is $L_a$-Lipschitz. Define the reweighted laws
\[
    \mathcal R_a(Q)(dx)
    :=
    \frac{a(x)}{\EE_Q[a(X)]}Q(dx),
    \qquad
    \mathcal R_a(\widehat Q)(dx)
    :=
    \frac{a(x)}{\EE_{\widehat Q}[a(X)]}\widehat Q(dx).
\]
Then
\[
    \WW_2(\mathcal R_a(Q),\mathcal R_a(\widehat Q))
    \leq
    \sqrt{
        \frac{2\Cnorm(1+2\Cnorm L_a)}{a_0}
        \WW_2(Q,\widehat Q)
    }.
\]
\end{lemma}

\begin{proof}
We first prove a Wasserstein-$1$ bound. By Kantorovich duality, it suffices to consider
$1$-Lipschitz functions $\varphi$ with $\varphi(0)=0$. On $\CB_{d,2}(\Cnorm)$, such functions
satisfy $|\varphi|\leq \Cnorm$.

Let
\[
    Z_Q:=\EE_Q[a(X)],
    \qquad
    Z_{\widehat Q}:=\EE_{\widehat Q}[a(X)].
\]
Since $a\geq a_0$, both normalizers are at least $a_0$. For any such $\varphi$,
\begin{align*}
&
\left|
    \EE_{\mathcal R_a(Q)}[\varphi]
    -
    \EE_{\mathcal R_a(\widehat Q)}[\varphi]
\right|
\\
&\qquad
=
\left|
    \frac{\EE_Q[a\varphi]}{Z_Q}
    -
    \frac{\EE_{\widehat Q}[a\varphi]}{Z_{\widehat Q}}
\right|
\\
&\qquad
\leq
\frac{
    |\EE_Q[a\varphi]-\EE_{\widehat Q}[a\varphi]|
}{Z_Q}
+
|\EE_{\widehat Q}[a\varphi]|
\left|
    \frac{1}{Z_Q}-\frac{1}{Z_{\widehat Q}}
\right|.
\end{align*}
The function $a\varphi$ has Lipschitz constant at most
\[
    \Lip(a\varphi)
    \leq
    \Lip(a)\norm{\varphi}_\infty+\norm{a}_\infty\Lip(\varphi)
    \leq
    \Cnorm L_a+1.
\]
Thus
\[
    |\EE_Q[a\varphi]-\EE_{\widehat Q}[a\varphi]|
    \leq
    (1+\Cnorm L_a)\WW_1(Q,\widehat Q).
\]
Also,
\[
    |\EE_{\widehat Q}[a\varphi]|
    \leq
    \Cnorm Z_{\widehat Q},
\]
and
\[
    |Z_Q-Z_{\widehat Q}|
    =
    |\EE_Q[a]-\EE_{\widehat Q}[a]|
    \leq
    L_a\WW_1(Q,\widehat Q).
\]
Therefore
\[
    \left|
        \frac{1}{Z_Q}-\frac{1}{Z_{\widehat Q}}
    \right|
    =
    \frac{|Z_Q-Z_{\widehat Q}|}{Z_QZ_{\widehat Q}}
    \leq
    \frac{L_a}{a_0 Z_{\widehat Q}}\WW_1(Q,\widehat Q).
\]
Combining these bounds gives
\[
    \left|
        \EE_{\mathcal R_a(Q)}[\varphi]
        -
        \EE_{\mathcal R_a(\widehat Q)}[\varphi]
    \right|
    \leq
    \frac{1+2\Cnorm L_a}{a_0}
    \WW_1(Q,\widehat Q).
\]
Taking the supremum over $\varphi$ gives
\[
    \WW_1(\mathcal R_a(Q),\mathcal R_a(\widehat Q))
    \leq
    \frac{1+2\Cnorm L_a}{a_0}
    \WW_1(Q,\widehat Q).
\]
Since all measures are supported on $\CB_{d,2}(\Cnorm)$, \cref{lem:w1-to-w2-unit} implies
\[
    \WW_2(\mathcal R_a(Q),\mathcal R_a(\widehat Q))
    \leq
    \sqrt{
        2\Cnorm\,
        \WW_1(\mathcal R_a(Q),\mathcal R_a(\widehat Q))
    }.
\]
Finally, $\WW_1(Q,\widehat Q)\leq \WW_2(Q,\widehat Q)$, which proves the claim.
\end{proof}

\section{Proofs and supporting lemmas for \cref{sec:wasserstein}} 
\label{a:wasserstein}

\begin{proof}[Proof of \cref{cor:w2-map}]

Let \(\Gamma(P)\) denote the set of probability measures \(\gamma\) on
\(\mathcal X\times \mathcal X\) whose second marginal is \(P\). We write
\((X,Y)\sim\gamma\), where \(Y\sim P\).
Then, we have: 
\[
    \sup_{Q}\EE_Q[r] - \lambda\WW_2^2(Q, P)
    =
    \sup_{\gamma\in\Gamma(P)}
    \EE_\gamma[r(X)] - \lambda \EE_\gamma \norm{X-Y}^2.
\]
By disintegration, any \(\gamma\in\Gamma(P)\) can be written as
\[
    \gamma(dx,dy)=K_y(dx)P(dy),
\]
where \(y\mapsto K_y\) is a stochastic kernel on \(\mathcal X\). Therefore the problem is
equivalently
\[
    \sup_K
    \left\{
    \int_{\mathcal X}\int_{\mathcal X} r(x)\,K_y(dx)\,P(dy)
    - \lambda
    \int_{\mathcal X}\int_{\mathcal X}\|x-y\|^2\,K_y(dx)\,P(dy)
    \right\}.
\]

For fixed \(y\), write
\[
    h_y(x):=r(x)-\lambda\|x-y\|^2.
\]
Since \(K_y\) ranges over all probability measures on \(\mathcal X\), we have
\[
    \int h_y(x)\,K_y(dx)\leq \sup_{x\in\mathcal X}h_y(x)
\]
for every \(K_y\). Conversely, choosing \(K_y=\delta_{x^\star}\) for a maximizer
\(x^\star\in\argmax_x h_y(x)\) achieves equality. If the maximum is not attained, the same
identity holds with supremum by choosing \(\varepsilon\)-maximizers. Hence
\[
    \sup_{K_y}
    \int_{\mathcal X}
    \left(r(x)-\lambda\|x-y\|^2\right)K_y(dx)
    =
    \sup_{x\in\mathcal X}
    \left\{r(x)-\lambda\|x-y\|^2\right\}.
\]

Thus
\[
    \sup_K
    \int\!\!\int
    \left(r(x)-\lambda\|x-y\|^2\right)K_y(dx)P(dy)
    =
    \EE_{Y\sim P}
    \left[
        \sup_{x\in\mathcal X}
        \left\{r(x)-\lambda\|x-Y\|^2\right\}
    \right]
\]
and hence
\[\sup_{Q}\EE_Q[r] - \lambda\WW_2^2(Q, P) = \EE_{Y\sim P}
    \left[
        \sup_{x\in\mathcal X}
        \left\{r(x)-\lambda\|x-Y\|^2\right\}\right].\]
Let $Y \sim P$ and $X := T_\lambda(Y)$, so that $X \sim Q_\lambda$. Then
\begin{align*}
\EE_{Q_\lambda}[r] - \lambda\WW_2^2(Q_\lambda,P)
&\geq \EE[r(X)] - \lambda \EE[\|X-Y\|^2] \\ 
&= \EE_{Y \sim P}[r(T_\lambda(Y)) - \lambda \|T_\lambda(Y) - Y\|^2] \\
&= \EE_{Y \sim P}\left[\sup_{x \in \mathcal X}\{r(x) - \lambda \|x-Y\|^2\}\right].
\end{align*}
Thus, $Q_\lambda$ achieves the supremum.
\end{proof}

\begin{lemma}[Using an approximate proximal oracle]
\label{cor:concave-wasserstein-sampler}
Suppose $\widehat T_\lambda$ is an approximate proximal oracle satisfying
\[
    \sup_{y\in\mathcal X}
    \|\widehat T_\lambda(y)-T_\lambda(y)\|
    \le \alpha.
\]
Then the output law $\widehat Q_\lambda := (\widehat T_\lambda)_\#P$ satisfies
\[
   \WW_2(\widehat Q_\lambda,Q_\lambda)\le \alpha.
\]
\end{lemma}

\begin{proof}
Let $Y\sim P$, and let
\[
    X=T_\lambda(Y),
    \qquad
    \widehat X=\widehat T_\lambda(Y).
\]
Then $X\sim Q_\lambda$ and $\widehat X\sim \widehat Q_\lambda$, and
\[
    \|\widehat X-X\|
    =
    \|\widehat T_\lambda(Y)-T_\lambda(Y)\|
    \le \alpha
\]
almost surely. Therefore
\[
   \WW_2^2(\widehat Q_\lambda,Q_\lambda)
    \le
    \EE\|\widehat X-X\|^2
    \le
    \alpha^2,
\]
which proves the claim.
\end{proof}

\begin{proof}[Proof of Lemma \ref{lem:wasserstein-convex-sampling-hard}]
We reduce from MAX-CUT. Let \(G=([d],E)\) be an unweighted graph, and let
\(L_G\succeq 0\) be its graph Laplacian. Set
\[
    \mathcal X=[-1,1]^d,
    \qquad
    P=\delta_0,
\]
and define the convex quadratic reward
\[
    r_G(x):=x^\top(L_G+I)x.
\]
The base distribution \(P=\delta_0\) has a trivial diffusion-score oracle: its noised distribution is
Gaussian, so the score can be evaluated explicitly. For any \(Q\) supported on \(\mathcal X\),
\[
   \WW_2^2(Q,\delta_0)=\EE_Q\|X\|^2\leq d.
\]
Thus the Wasserstein constraint with radius \(\sqrt d\) is automatically satisfied, and the optimization problem is
equivalent to
\[
    \sup_{Q\in\Delta(\mathcal X)} \EE_Q[r_G(X)].
\]
Since the objective is linear in \(Q\), every optimizer is supported on the maximizers of \(r_G\)
over \(\mathcal X\).

We now characterize these maximizers. For \(s\in\{\pm1\}^d\),
\[
    s^\top L_G s
    =
    \sum_{(i,j)\in E}(s_i-s_j)^2.
\]
Moreover,
\[
    r_G(s)=s^\top L_Gs+d.
\]
Since \(x\mapsto x^\top(L_G+I)x\) is strictly convex in each coordinate when the others are fixed,
any maximizer over the box \([-1,1]^d\) lies at a vertex \(s\in\{\pm1\}^d\). Therefore the
maximizers of \(r_G\) over \(\mathcal X\) are exactly the vertices corresponding to maximum cuts
of \(G\).

Now suppose we have a sampler producing \(\widehat Q\) with
\(W_2(\widehat Q,Q^\star)\leq 1/4\), where \(Q^\star\) is an optimal Wasserstein tilt. Couple
\(\widehat X\sim\widehat Q\) and \(X^\star\sim Q^\star\) so that
\[
    \EE\|\widehat X-X^\star\|^2\leq \frac{1}{16}.
\]
Since \(Q^\star\) is supported on maximum-cut vertices, \(X^\star\in\{\pm1\}^d\) almost surely.
By Markov's inequality,
\[
    \Pr\!\left[\|\widehat X-X^\star\|\geq 1\right]\leq \frac{1}{16}.
\]
On the event \(\|\widehat X-X^\star\|<1\), coordinatewise rounding recovers the same vertex:
\[
    \operatorname{sign}(\widehat X)=X^\star.
\]
Hence, with probability at least \(15/16\), rounding the sampler output gives a maximum cut of
\(G\).

Thus, given a MAX-CUT decision instance \((G,K)\), we run the sampler, round the output to
\(\operatorname{sign}(\widehat X)\), and check whether the resulting cut has value at least \(K\).
If the instance is a NO instance, no rounded vector can have cut value at least \(K\). If it is a YES
instance, the procedure outputs such a cut with probability at least \(15/16\). This gives a
randomized polynomial-time algorithm for MAX-CUT, and hence implies
\(\mathrm{NP}\subseteq\mathrm{BPP}\).
\end{proof}

\begin{proof}[Proof of \cref{thm:lowrank-wasserstein}]
We analyze \cref{alg:lowrank-wasserstein}. Let \(\widehat P\) denote the law of \(\widehat Y\). By
\cref{thm:chen-base-sampling}, and since projection onto
\(\CB_{d,2}(\Cnorm)\) is nonexpansive, with probability at least \(1-\delta\),
\[
   \WW_2(\widehat P,P)\le \epsilon_P,
    \qquad
    \supp(\widehat P)\subseteq \CB_{d,2}(\Cnorm).
\]
We condition on this event.

Fix \(y\in\CB_{d,2}(\Cnorm)\). Write
\[
    A=U\Sigma V_1^\top,
\]
where \(U\in\RR^{k\times r_A}\), \(\Sigma\in\RR^{r_A\times r_A}\), and
\(V_1\in\RR^{d\times r_A}\) has orthonormal columns. Complete \(V_1\) to an orthogonal basis by
\(V_0\in\RR^{d\times(d-r_A)}\). Every \(x\in\RR^d\) can be written uniquely as
\[
    x=V_1u+V_0w,
    \qquad
    u\in\RR^{r_A},
    \quad
    w\in\RR^{d-r_A}.
\]
Moreover,
\[
    Ax=U\Sigma u.
\]
The constraint \(x\in\CB_{d,2}(\Cnorm)\) is equivalent to
\[
    \|u\|_2^2+\|w\|_2^2\le \Cnorm^2.
\]
Writing
\[
    u_y:=V_1^\top y,
    \qquad
    w_y:=V_0^\top y,
\]
we have
\[
    \|x-y\|_2^2
    =
    \|u-u_y\|_2^2+\|w-w_y\|_2^2.
\]
For fixed \(u\), the reward \(f(Ax)=f(U\Sigma u)\) does not depend on \(w\). Thus, among all
\(w\) satisfying
\[
    \|w\|_2\le \sqrt{\Cnorm^2-\|u\|_2^2},
\]
the optimal choice is the projection of \(w_y\) onto this ball. Therefore the pointwise proximal
value
\[
    V(y)
    :=
    \sup_{x\in\CB_{d,2}(\Cnorm)}
    \left\{
        f(Ax)-\lambda\|x-y\|_2^2
    \right\}
\]
can be written as
\[
    V(y)
    =
    \sup_{\|u\|_2\le \Cnorm}
    \Phi_y(u),
\]
where
\[
    \Phi_y(u)
    :=
    f(U\Sigma u)-\lambda
    \left[
        \|u-u_y\|_2^2+
        \bigl(\|w_y\|_2-\rho(u)\bigr)_+^2
    \right],
\]
and
\[
    \rho(u):=\sqrt{\Cnorm^2-\|u\|_2^2}.
\]
The algorithm computes exactly this reduced objective on a net, and returns the corresponding
point \(x(u)\in\CB_{d,2}(\Cnorm)\).

We next bound the modulus of continuity of \(\Phi_y\). Since \(f\) is \(L\)-Lipschitz on
\(\CB_{k,2}(S\Cnorm)\) and \(\|U\Sigma\|_{\op}=\|A\|_{\op}=S\), the map
\(u\mapsto f(U\Sigma u)\) is \(LS\)-Lipschitz. The term
\[
    u\mapsto \|u-u_y\|_2^2
\]
is \(4\Cnorm\)-Lipschitz on \(\CB_{r_A,2}(\Cnorm)\). It remains to control the radial term
\[
    \psi_y(u):=\bigl(\|w_y\|_2-\rho(u)\bigr)_+^2.
\]
For \(u,u'\in\CB_{r_A,2}(\Cnorm)\),
\[
    |\rho(u)-\rho(u')|
    \le
    \sqrt{\left|\rho(u)^2-\rho(u')^2\right|}
    =
    \sqrt{\left|\|u'\|_2^2-\|u\|_2^2\right|}
    \le
    \sqrt{2\Cnorm\|u-u'\|_2}.
\]
The function \(t\mapsto(\|w_y\|_2-t)_+^2\) is \(2\Cnorm\)-Lipschitz on
\([0,\Cnorm]\). Hence
\[
    |\psi_y(u)-\psi_y(u')|
    \le
    2\Cnorm\sqrt{2\Cnorm\|u-u'\|_2}.
\]
Combining the preceding displays, for all \(u,u'\in\CB_{r_A,2}(\Cnorm)\),
\[
    |\Phi_y(u)-\Phi_y(u')|
    \le
    (LS+4\lambda\Cnorm)\|u-u'\|_2
    +
    2\lambda\Cnorm\sqrt{2\Cnorm\|u-u'\|_2}.
\]

Let \(u^\star\) maximize \(\Phi_y\) over \(\CB_{r_A,2}(\Cnorm)\). Since \(\mathcal N_h\) is an
\(h\)-net of \(\CB_{r_A,2}(\Cnorm)\), there exists \(\bar u\in\mathcal N_h\) such that
\[
    \|\bar u-u^\star\|_2\le h.
\]
By the choice of \(h\),
\[
    (LS+4\lambda\Cnorm)h\le \epsilon/6,
\]
and
\[
    2\lambda\Cnorm\sqrt{2\Cnorm h}\le \epsilon/6.
\]
Therefore
\[
    \Phi_y(\bar u)\ge \Phi_y(u^\star)-\epsilon/3=V(y)-\epsilon/3.
\]
Since the algorithm chooses a maximizer \(\widehat u\) over the grid,
\[
    \Phi_y(\widehat u)\ge V(y)-\epsilon/3.
\]
Applying this with \(y=\widehat Y\), and using the fact that the returned point \(\widehat X\)
satisfies
\[
    f(A\widehat X)-\lambda\|\widehat X-\widehat Y\|_2^2
    =
    \Phi_{\widehat Y}(\widehat u),
\]
we obtain
\[
    \EE
    \left[
        f(A\widehat X)-\lambda\|\widehat X-\widehat Y\|_2^2
    \right]
    \ge
    \EE_{\widehat Y\sim\widehat P}[V(\widehat Y)]-\epsilon/3.
\]

We now transfer from \(\widehat P\) back to \(P\). For \(y,y'\in\CB_{d,2}(\Cnorm)\), let
\(x_y\) be a maximizer defining \(V(y)\). Then
\[
\begin{aligned}
    V(y)-V(y')
    &\le
    f(Ax_y)-\lambda\|x_y-y\|_2^2
    -
    \left(
        f(Ax_y)-\lambda\|x_y-y'\|_2^2
    \right) \\
    &=
    \lambda
    \left(
        \|x_y-y'\|_2^2-\|x_y-y\|_2^2
    \right).
\end{aligned}
\]
Since \(x_y,y,y'\in\CB_{d,2}(\Cnorm)\),
\[
    |V(y)-V(y')|
    \le
    4\lambda\Cnorm\|y-y'\|_2.
\]
Thus \(V\) is \(4\lambda\Cnorm\)-Lipschitz on the support region. Since
\[
    W_1(\widehat P,P)\le\WW_2(\widehat P,P)\le \epsilon_P,
\]
we have
\[
    \EE_{\widehat P}[V]-\EE_P[V]\ge -4\lambda\Cnorm\epsilon_P.
\]
By the pointwise representation of the squared-Wasserstein penalized objective,
\[
    \OPT_\lambda(P)=\EE_{Y\sim P}[V(Y)].
\]
Therefore
\[
    \EE
    \left[
        f(A\widehat X)-\lambda\|\widehat X-\widehat Y\|_2^2
    \right]
    \ge
    \OPT_\lambda(P)-\epsilon/3-4\lambda\Cnorm\epsilon_P.
\]

Let \((Y,\widehat Y)\) be a coupling of \(P\) and \(\widehat P\) satisfying
\[
    \left(\EE\|Y-\widehat Y\|_2^2\right)^{1/2}\le \epsilon_P.
\]
Using the same conditional randomness that maps \(\widehat Y\) to \(\widehat X\), this gives a
coupling between \(Y\sim P\) and \(\widehat X\sim\widehat Q\). Thus
\[
   \WW_2^2(\widehat Q,P)\le \EE\|\widehat X-Y\|_2^2.
\]
Because \(\widehat X,\widehat Y,Y\in\CB_{d,2}(\Cnorm)\),
\[
\begin{aligned}
&
\left|
    \|\widehat X-Y\|_2^2
    -
    \|\widehat X-\widehat Y\|_2^2
\right| \\
&\qquad\le
    4\Cnorm\|Y-\widehat Y\|_2.
\end{aligned}
\]

Consequently,
\[
\begin{aligned}
    \EE_{X\sim\widehat Q}[f(AX)]-\lambda\WW_2^2(\widehat Q,P)
    &\ge
    \EE
    \left[
        f(A\widehat X)-\lambda\|\widehat X-Y\|_2^2
    \right] \\
    &\ge
    \EE
    \left[
        f(A\widehat X)-\lambda\|\widehat X-\widehat Y\|_2^2
    \right]
    -
    4\lambda\Cnorm\epsilon_P \\
    &\ge
    \OPT_\lambda(P)-\epsilon/3-8\lambda\Cnorm\epsilon_P.
\end{aligned}
\]
By the choice
\[
    \epsilon_P=\frac{\epsilon}{24\lambda\Cnorm},
\]
we conclude
\[
    \EE_{X\sim\widehat Q}[f(AX)]
    -
    \lambda\WW_2^2(\widehat Q,P)
    \ge
    \OPT_\lambda(P)-\epsilon.
\]

It remains to bound the runtime. The grid size satisfies
\[
    |\mathcal N_h|
    \le
    \left(
        1+\frac{2\Cnorm}{h}
    \right)^{r_A}
    =
    M_W.
\]
Each grid point requires one value-oracle query to \(f\) and polynomial-time arithmetic in
\(d\) and \(k\). The call to \textsc{BaseSampler} runs in time
\[
    \poly(d,\Cnorm,\epsilon_P^{-1},\log(1/\delta)).
\]
Since
\[
    \epsilon_P^{-1}
    =
    \frac{24\lambda\Cnorm}{\epsilon},
\]
the claimed runtime follows.
\end{proof}

\begin{algorithm}[t]
\caption{\textsc{LowRankWassersteinTiltSampler}: low-rank Wasserstein tilt}
\label{alg:lowrank-wasserstein}
\begin{algorithmic}[1]
  \State \textbf{Input:} score oracle \((s_\sigma)_\sigma\) for \(P\), matrix \(A\), value oracle for \(f\), parameters \(L,\lambda,\Cnorm,\epsilon,\delta\).
  \State Compute a compact SVD
  \[
      A=U\Sigma V_1^\top,
  \]
  where \(U\in\RR^{k\times r_A}\) and \(V_1\in\RR^{d\times r_A}\) have orthonormal columns. Let \(V_0\in\RR^{d\times(d-r_A)}\) complete \(V_1\) to an orthogonal basis.
  \State Set
  \[
      h\gets
      \min\left\{
          \frac{\epsilon}{6(LS+4\lambda\Cnorm)},
          \frac{\epsilon^2}{288\lambda^2\Cnorm^3}
      \right\},
      \qquad
      \epsilon_P\gets \frac{\epsilon}{24\lambda\Cnorm}.
  \]
  \State Draw
  \[
      \widetilde Y\leftarrow
      \textsc{BaseSampler}((s_\sigma)_\sigma,\epsilon_P,\delta,\Cnorm),
  \]
  and set
  \[
      \widehat Y\gets \Pi_{\CB_{d,2}(\Cnorm)}(\widetilde Y).
  \]
  \State Write the coordinates of \(\widehat Y\) in the SVD basis:
  \[
      u_{\widehat Y}:=V_1^\top \widehat Y,
      \qquad
      w_{\widehat Y}:=V_0^\top \widehat Y.
  \]
  \State Construct an \(h\)-net \(\mathcal N_h\) of \(\CB_{r_A,2}(\Cnorm)\).
  \For{\(u\in\mathcal N_h\)}
    \State Set
    \[
        \rho(u):=\sqrt{\Cnorm^2-\|u\|_2^2}.
    \]
    \State Let \(w(u)\) be the Euclidean projection of \(w_{\widehat Y}\) onto
    \(\CB_{d-r_A,2}(\rho(u))\), i.e.
    \[
        w(u)
        :=
        \begin{cases}
        w_{\widehat Y}, & \|w_{\widehat Y}\|_2\le \rho(u),\\[3pt]
        \rho(u)\,w_{\widehat Y}/\|w_{\widehat Y}\|_2, & \|w_{\widehat Y}\|_2>\rho(u).
        \end{cases}
    \]
    \State Define
    \[
        x(u):=V_1u+V_0w(u).
    \]
    \State Compute
    \[
        \Phi_{\widehat Y}(u)
        :=
        f(U\Sigma u)-\lambda\|x(u)-\widehat Y\|_2^2.
    \]
  \EndFor
  \State Let
  \[
      \widehat u\in\argmax_{u\in\mathcal N_h}\Phi_{\widehat Y}(u).
  \]
  \State \Return \(x(\widehat u)\).
\end{algorithmic}
\end{algorithm}

\end{appendices}

\ifneurips
\newpage
\input{checklist.tex}
\fi

\end{document}